\documentclass{article}

\PassOptionsToPackage{numbers}{natbib}
\usepackage[preprint]{neurips_2023}

\usepackage[utf8]{inputenc} 
\usepackage[T1]{fontenc}    

\PassOptionsToPackage{unicode}{hyperref}
\PassOptionsToPackage{naturalnames}{hyperref}
\usepackage{hyperref}       

\usepackage{url}            
\usepackage{booktabs}       
\usepackage{amsfonts}       
\usepackage{nicefrac}       
\usepackage{microtype}      
\usepackage[dvipsnames]{xcolor}         

\usepackage[ruled, linesnumbered]{algorithm2e}
\usepackage{algorithmic}

\hypersetup{
    colorlinks = true,
    citecolor = blue,
    linkcolor = blue
}

\usepackage{pythonhighlight}

\usepackage{amsmath,amssymb,amsthm,mathtools}
\usepackage{dsfont}
\usepackage{xfrac}

\newcommand{\ch}[1]{{{\color{magenta} [ch:] {#1}}}}

\newtheorem{assumption}{Assumption}
\newtheorem{theorem}{Theorem}

\renewcommand{\leq}{\leqslant}
\renewcommand{\geq}{\geqslant}
\renewcommand{\phi}{\varphi}
\renewcommand{\epsilon}{\varepsilon}
\renewcommand{\enspace}{\,}

\newcommand{\R}{\mathbb{R}}
\newcommand{\cO}{\mathcal{O}}
\newcommand{\E}{\mathbb{E}}

\newcommand{\eps}{\varepsilon}
\newcommand{\norm}[1]{\left\|#1\right\|}
\newcommand{\scalar}[2]{\left\langle#1,\,#2\right\rangle}
\newcommand{\bbR}{\mathbb{R}}
\newcommand{\parent}[1]{\left(#1\right)}
\DeclareMathOperator{\Proj}{Proj}

\newcommand{\e}{\mathrm{e}}

\usepackage{xspace}
\newcommand{\FreeAdaGrad}{\textsc{Free AdaGrad}\xspace}
\newcommand{\AdaGrad}{\textsc{AdaGrad}\xspace}
\newcommand{\Oracle}{\textsc{Oracle}\xspace}
\usepackage{enumitem}

\usepackage{pifont}
\newcommand{\Yes}{\ding{51}}%
\newcommand{\No}{\ding{55}}%
\newcommand{\Maybe}{?}%

\usepackage[textwidth=2.0cm, textsize=tiny]{todonotes} 
\newcommand{\evg}[2][]{\todo[color=yellow!20,#1]{{\bf Evg:} #2}}

\newcommand{\pa}[1]{\left(#1\right)}
\newcommand{\cro}[1]{\left[#1\right]}
\newcommand{\ac}[1]{\left\{#1\right\}}

\DeclareMathOperator*{\argmin}{arg\,min}

\newtheorem{lemma}{Lemma}

\newtheorem{corollary}{Corollary}

\usepackage[most]{tcolorbox}
\newtcolorbox{nbox}[1][]{
  enhanced,
  fonttitle=\scshape,
  #1
}

\title{Parameter-free projected gradient descent}

\author{Evgenii Chzhen \qquad Christophe Giraud \qquad Gilles Stoltz\\
Universit{\'e} Paris-Saclay, CNRS, Laboratoire de mathématiques d'Orsay, 91405, Orsay, France \\
\texttt{\{evgenii.chzhen,\,\,christophe.giraud,\,\,gilles.stoltz\}@universite-paris-saclay.fr}
}

\begin{document}

\maketitle

\begin{abstract}
  We consider the problem of minimizing a convex function over a closed convex  set, with Projected Gradient Descent (PGD).
  We propose a fully parameter-free version of AdaGrad, which is adaptive to the distance between the initialization and the optimum, and to the sum of the square norm of the subgradients.  Our algorithm is able to handle projection steps, does not involve restarts, reweighing along the trajectory or additional gradient evaluations compared to the classical PGD. It also fulfills optimal rates of convergence for cumulative regret up to logarithmic factors.
  We provide an extension of our approach to stochastic optimization and conduct numerical experiments supporting the developed theory.
\end{abstract}

\section{Introduction}
In this work we study the problem of minimizing a convex  function $f$ over a closed, possibly unbounded, convex set $\Theta \subseteq \bbR^d$. Our main goal is to provide a variant of AdaGrad~\cite{streeter2010less,duchi2011adaptive} which is adaptive to the distance $\|x_1 - x_*\|$ between the initialization $x_{1} \in \Theta$ and a minimizer $x_{*} \in \Theta$, which is assumed to exist.
More precisely, we provide a Projected Gradient Descent (PGD) algorithm of the form
\begin{align*}
  x_{t+1} = \Proj_{\Theta}\pa{x_t - \eta_t g_t}\quad\text{with}\quad \eta_t = \frac{2^{k_t}}{H\big(\sum_{s\leq t}\|g_{s}\|^2\big)}\enspace,
\end{align*}
where $g_t \in \partial f(x_t)$ is a sub-gradient of $f$ at $x_t$, $\Proj_{\Theta}(\cdot)$ is the Euclidean projection operator onto closed convex $\Theta$, $H(x) = \sqrt{(x+1)\log(\e(1 + x))}$ and $k_t$ is an automatically tuned sequence by Algorithm~\ref{algo:PGDloglogMath}.
Unlike recent works on the subject~\cite{defazaio2023learning,carmon2023making}, we provide bounds on the cumulative regret of the form
\begin{align*}
    R_T : = \sum_{t = 1}^T \big(f(x_t) - f(x_*)\big)\enspace,
\end{align*}
where $x_*$ is any minimizer of $f$ over $\Theta$. Using standard online-to-batch conversion, we also have by convexity $f(\bar{x}_T) - f(x_*) \leq R_T / T$, for $\bar{x}_T$ being the average of $x_1, \ldots, x_T$.

In the classical case where $f$ is assume to be $L$-Lipschitz, it is well known that setting $\eta_t = \tfrac{\|x_1 - x_*\|}{L \sqrt{T}}$ gives the optimal rate of convergence~\cite{nesterov2018lectures}:
\begin{align*}
    R_T \leq \|x_1 - x_*\|L \sqrt{T}\enspace.
\end{align*}
However, such a choice requires $f$ to be Lipschitz, and the knowledge of three quantities: \emph{1)} distance to the optimum $\|x_1 - x_*\|$; \emph{2)} Lipschitz constant $L$; \emph{3)} optimization horizon $T$.
Should the distance $\|x_1 - x_*\|$ be known, one could set $\eta_t = \tfrac{\|x_1 - x_*\|}{ \sqrt{\sum_{s = 1}^t\|g_s\|^2}}$, resulting in \textsc{AdaGrad} algorithm~\cite{streeter2010less,duchi2011adaptive}. For this choice of $\eta_{t}$, without Lipschitz assumption, we have the upper bound on the regret
\begin{align}\label{eq:AdaGrad:regret}
    R_T \leq c\|x_1 - x_*\| \sqrt{\sum_{t = 1}^T \|g_t\|^2}\enspace.
\end{align}
In practice the distance $\|x_1 - x_*\|$ is unknown. When an upper bound $D_{*}$ on $\|x_1 - x_*\|$ is available, typically the diameter of $\Theta$ when $\Theta$ is bounded, $\|x_1 - x_*\|$ can be replaced by $D_{*}$ in $\eta_{t}$. The \textsc{AdaGrad} algorithm then fulfills (\ref{eq:AdaGrad:regret}) with $\|x_1 - x_*\|$ replaced by $D_{*}$. This bound can be very sub-optimal yet, when  $\|x_1 - x_*\|$ is much smaller than $D_{*}$. Worse,  when $\Theta$ is unbounded, no bound $D_{*}$ on $\|x_1 - x_*\|$ is available, without additional information.

Our objective is to provide a variant of the \textsc{AdaGrad} step-size tuning, not requiring $f$ to be Lipschitz, nor any knowledge on  $\|x_1 - x_*\|$ or $T$, while still fulfilling the regret bound (\ref{eq:AdaGrad:regret}) up to a log factor.
Our contribution can be placed alongside the ever expanding literature of parameter-free optimization algorithm~\cite{defazaio2023learning,
carmon2023making,
pmlr-v99-cutkosky19a,
mcmahan2012no,
mcmahan2014unconstrained,
pmlr-v125-mhammedi20a,
orabona2021parameter,
orabona2017training,
zhang2022pde,
jacobsen2022parameter,
orabona2016coin}, discussed below Theorem~\ref{thm:main_intro}.\\
\textbf{Main contributions.} Let us describe our three main contributions
\begin{enumerate}[noitemsep,topsep=0ex,leftmargin=.5cm]
\item we propose a simple tuning of PGD, we call \FreeAdaGrad, with no line-search, no cold-restart, no gradient transformation, and no computations of extra gradients;
\item we handle any finite convex function $f$ (no Lipschitz condition), over any possibly unbounded constraint set $\Theta$;
\item we provide regret bounds like {$R_T =\widetilde{\cO}\big((\|x_1 - x_*\|+1) \sqrt{1+\sum_{t \leq T} \|g_t\|^2}\big)$}, where $\tilde \cO$ hides log-factor,  but no additional terms.
\end{enumerate}
We also partially extend our results to the Stochastic Gradient Descent setting.

\textbf{Notation.}
For any $a, b \in \bbR$ we denote by $a \vee b$ (\emph{resp.} $a \wedge b$) the largest (\emph{resp.} the smallest) of the two. We denote by $\|\cdot\|$ the Euclidean norm and by $\scalar{\cdot}{\cdot}$ the standard inner product in $\bbR^d$. For $\Theta \subset \bbR^d$, we denote by $\Proj_{\Theta}(\cdot)$ the Euclidean projection operator onto $\Theta$. We denote by $\log_2(\cdot)$ and $\log(\cdot)$ the base $2$ and the natural logarithms respectively. The base of the natural logarithm is denoted by $\e$.

\section{Main result}
\label{sec:main_result}

\begin{algorithm}[t!]
\DontPrintSemicolon
\caption{\FreeAdaGrad}\label{algo:PGDloglogMath}
\SetKwInput{Input}{Input}
   \SetKwInOut{Output}{Output}
   \SetKwInput{Initialization}{Initialization}
   \Input{$x_1 \in \bbR^d, \Theta \subset \bbR^d, \gamma_0 > 0$}
   \Initialization{$\Gamma^2_1 = 0, k_0 = 1, S_0 = 0,
   \gamma_k = \gamma_0 2^k \text{ for $k \geq 1$}$}

\For{$t \geq 1$}{

    $g_t \in \partial f(x_t)$  \tcp*[r]{{\small get subgradient}}

   $S_t = S_{t - 1} + \|g_t\|^2$ \tcp*[r]{{\small cumulative grad-norm}}

  {$h_t = \sqrt{(S_t+1)\log(\e(1 + S_t))}$} \label{alg:line_h}\tcp*[r]{{\small update $h_t \geq h_{t-1}$}}



    $\displaystyle{B_{t + 1}(k) = \frac{2\gamma_{k}}{\sqrt{k}} + \sqrt{\Gamma^2_{t}+ \gamma_{k}^2 {\|g_{t}\|^2/ h^2_{t}}}}$\label{alg:B} \tcp*[r]{\small define the threshold}

    $x^+_{t}(k) = \Proj_{\Theta}\pa{x_t - \frac{\gamma_k}{h_t}g_t}$ \tcp*[r]{\small probing step}

    $k_t = \min\left\{k \geq k_{t-1} \,:\, \|x^+_{t}(k) - x_1\| \leq B_{t + 1}(k)\right\}$ \label{alg:line_bound}\tcp*[r]{\small find the step size}

    $x_{t + 1} = x^+(k_t)$ \tcp*[r]{\small make the step}

   $\Gamma_{t+1}^2 = \Gamma_t^2 + \gamma_{k_t}^2\frac{\|g_t\|^2}{h^2_t}$ \tcp*[r]{\small update $\Gamma^2_t$}

    }

\Output{Trajectory $(x_t)_{t \geq 1}$}

\end{algorithm}

We make the following assumption, which is necessary for the meaningful treatment of the problem.

\begin{assumption}
\label{ass:main}
    The set $\Theta \subseteq \bbR^d$ is closed convex, $\mathcal{D} \subseteq \bbR^d$ is open such that $\Theta \subseteq \mathcal{D}$.
    The function $f: \mathcal{D} \to \bbR$ is convex on $\Theta$ and there exists a bounded minimizer $x_* \in \argmin_{x \in \Theta} f(x)$.
\end{assumption}

Let us highlight that we do not assume that the subgradients $g_t$ are uniformly bounded, that is, we do not require $f$ to be globally Lipschitz.
This stays in contrast with the literature on online convex optimization (OCO).
Indeed, OCO lower bounds imply that without any prior knowledge a regret $\widetilde{\cO}(\|x_1 - x_*\|({\sum_{t \leq T}\|g_t\|^2})^{\sfrac{1}{2}})$ is not achievable~\cite{cutkosky2017online,pmlr-v99-cutkosky19a, pmlr-v125-mhammedi20a} and higher order terms either in $T$ or in $\|x_1 - x_*\|$ are necessary.
For example, we are able to handle $\Theta = [0, +\infty)^d$ and $f(x) = \sum_{i = 1}^n \exp(\|x - a_i\| / \sigma_i)$ for some $a_i \in \bbR^d$ and $\sigma_i > 0$.

The proposed method, that we call \FreeAdaGrad, is summarized in Algorithm~\ref{algo:PGDloglogMath}.
It consists in simple projected gradient steps at every round $t \geq 1$, but with additional cheap condition on Line~\ref{alg:line_bound} that is checked on every iteration. If the condition on this line is satisfied for $k=k_{t-1}$, then the algorithm makes (almost) the usual \AdaGrad step, otherwise, the step size is doubled and the condition is checked again.
We underline that the sequence of integers $\ac{k_{t}:t\geq 1}$ in Algorithm~\ref{algo:PGDloglogMath} is non-decreasing, and we prove in Eq.~\eqref{eq:bound_on_k_T} that it is upper-bounded by
$2 (\log_2(1+\|x_1 - x_*\| / \gamma_0) + 1)$.
Thus, there are only a finite number of doublings in the sequence. The only input parameter of the algorithm is $\gamma_0$ that can be seen as an initial lower-bound guess for $\|x_1 - x_*\|$ and can be taken arbitrary small only incurring additional $\sqrt{\log(1 / \gamma_0)}$ factor. Note that once the step-size is doubled at some time $t \geq 1$, Algorithm~\ref{algo:PGDloglogMath} continues the optimization from $x_t$ without cold-restart.

\begin{theorem}
  \label{thm:main_intro}
  Let Assumption~\ref{ass:main} be satisfied. Let $S_T = \sum_{t = 1}^T\|g_t\|^2$.
  For any $\gamma_0 > 0$, let $D_{\gamma_0} := \|x_1 - x_*\| \vee \gamma_0$\,, Algorithm~\ref{algo:PGDloglogMath} satisfies for some universal $c > 0$
  \begin{align*}
  {
    R_{T} \leq cD_{\gamma_0}\sqrt{(S_{T+1}+1)\log(1 + S_{T+1})\log(1 + D_{\gamma_0})}\log\log(1 + S_T)\enspace.}
  \end{align*}
\end{theorem}
{A large part of the literature on parameter-free optimization considers the context of online convex optimization with $L$-Lipschitz functions. A series of papers \cite{mcmahan2014unconstrained,zhang2022pde,jacobsen2022parameter,cutkosky2018black} have produced algorithms, mainly based on coin betting, enjoying regret bounds
$\cO(D_{\gamma_{0}}\sqrt{S_{T}\log(1+D_{\gamma_{0}}S_{T}/\gamma_{0})})$, up to lower order terms. Such a regret bound is not achievable in online convex optimization when $L$ is unknown as shown in~\cite{cutkosky2017online}. Some papers~\cite{pmlr-v99-cutkosky19a,pmlr-v125-mhammedi20a,jacobsen2022parameter} consider yet the case where $L$ is unknown, and provide regret bound including additional terms depending on $L$ and on higher order of $D_{\gamma_{0}}$.

If only the optimization error is of concern (and not the regret), bounds with log-factors replaced by log-log factors have been produced by~\cite{carmon2023making, defazaio2023learning}, breaking the barrier of online-to-batch conversion, but still requiring some knowledge about the Lipschitz constant $L$.
Relying on binary search,~\cite{carmon2023making} construct an adaptive algorithm for the problem of stochastic optimization, while~\cite{defazaio2023learning} provide an adaptive version of dual averaging and gradient descent algorithms, without allowing for projection step and requiring a careful weighted averaging along the trajectory to obtain the final solution. In contrast, we do not require Lipschitz condition, we handle the projection step and our bounds are valid for the usual average.

 The closer to us, in a setting of online convex optimization with $L$-Lipschitz functions,  \cite{mcmahan2012no} propose a tuning of GD (without projection) which is based on a doubling trick with cold-restarts and which requires the knowledge of $L$. This algorithm is shown to be adaptive to $\|x_{1}-x_{*}\|$ at the price of loosing a log factor
$\log(D_{\gamma_{0}}T/\gamma_{0})$ in the regret bound.
In Appendix~\ref{app:MS12}, we show that the algorithm of~\cite{mcmahan2012no} can be seen as a specific instantiation of Algorithm~\ref{algo:PGDloglogMath}, with the major difference that cold-restart are performed when doubling the step-size.

\section{Warm-up: simple analysis and intuition}\label{sec:warm-up}

Before proceeding to the analysis of the \textsc{Free AdaGrad} algorithm~\ref{algo:PGDloglogMath}, we explain the main ideas  behind our step-size scheme in the following simpler setup.

\begin{figure}[!tbp]
 \centering
   \includegraphics[width=\textwidth]{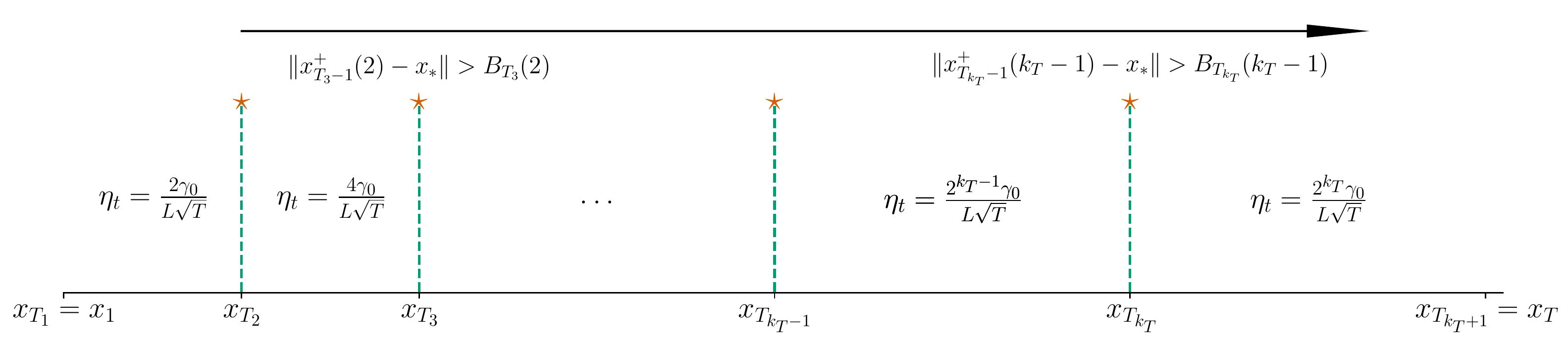}
   \caption{Schematic illustration of the algorithm in the simplest case.}\label{fig:schematic}
\end{figure}

\fbox{\parbox{\textwidth}{
{\bf Simple warm-up setup}\smallskip\\*
\ \ 1) the norm of the subgradients are uniformly bounded by some  known $L$, i.e. $\|g_{t}\|\leq L$,\smallskip\\*
\ \ 2) the optimization is unconstrained, i.e. $\Theta=\R^d$,\smallskip\\*
\ \ 3) the time horizon $T$ is fixed in advance.
}}

 In this case, we can replace $(h_t)_{t \geq 1}$ set on Line~\ref{alg:line_h} of Algorithm~\ref{algo:PGDloglogMath} by the constant sequence $h_t=L\sqrt{T}$, and  the choice $\gamma = \|x_1 - x_*\|$ is known to achieve the optimal rates for the regret $\|x_1 - x_*\|L\sqrt{T}$, see e.g.~\cite{nesterov2018lectures}. In this context, the overall  strategy of Algorithm~\ref{algo:PGDloglogMath} is to start from a small value $\gamma_{0}$ for $\gamma$, and then track $\norm{x_{t}-x_{1}}$ in order to detect if $\gamma<\|x_1 - x_*\|$. If so, $\gamma$ is doubled. The algorithm then increases the value $\gamma$ until reaching the level  $\norm{x_{t}-x_{1}}$.

In order to keep the analysis simple  in this warm-up section, we replace the threshold $B_{t+1}(k)$ Line~\ref{alg:B} of Algorithm~\ref{algo:PGDloglogMath} by $B^{\text{simple}}_{t+1}(k)=3\gamma_{k}$ (recall that $\gamma_k = \gamma_0 2^k$), what eventually leads to a slightly worse bound.
The gradient step and the step-size choice are then simply
\begin{equation}\label{warm-choice}
x^+_{t}(k)=x_{t}-{\gamma_{k}\over L\sqrt{T}}g_{t}\quad \text{and}\quad k_t = \min\left\{k \geq k_{t-1} \,:\, \|x^+_{t}(k) - x_1\| \leq 3\gamma_{k}\right\}\,.
\end{equation}
Below, we sketch the main arguments, and we refer to Appendix~\ref{app:warm-up} for all the details.\\
The first ingredient is the text-book decomposition using subgradient upper-bound: for any $k\geq 1$
\begin{align}\label{eq:one-step}
0\leq f(x_{t})-f(x_{*}) \leq \langle g_{t}, x_{t}-x_{*} \rangle
&= {{\gamma_{k}\over 2L\sqrt{T}} \|g_{t}\|^2 +{L\sqrt{T}\over 2 \gamma_{k}} \pa{\|x_{*}-x_{t}\|^2-\|x_{t}^+(k)-x_{*}\|^2}}\nonumber \\
& \leq {{\gamma_{k}L\over 2\sqrt{T}} +{L\sqrt{T}\over 2 \gamma_{k}} \pa{\|x_{*}-x_{t}\|^2-\|x_{t}^+(k)-x_{*}\|^2}}\enspace.
\end{align}
It follows from this bound, a one-step deviation upper-bound
\[\|x^+_{t}(k)-x_{*}\|^2\leq \|x_{t}-x_{*}\|^2+ {\gamma_{k}^2/ T}\enspace.\]
Summing this bound over $t$, we get a first important bound on the distance to optimum
\begin{align}
\|x^+_{t}(k)-x_{*}\|^2&\leq \|x_{1}-x_{*}\|^2+\sum_{s=1}^{t-1} {\gamma^2_{k_{s}}\over T} +{\gamma_{k}^2\over T}
\ \leq \ \|x_{1}-x_{*}\|^2 +{\gamma_{k}^2},\quad \text{for all}\ k\geq k_{t-1}\enspace,\label{eq:dist-opt}
\end{align}
and then another important bound on the distance to initialization
\begin{align}\label{eq:dist-init}
\|x^+_{t}(k)-x_{1}\|&\leq \|x_{1}-x_{*}\|+\|x^+_{t}(k)-x_{*}\| \ \leq\ 2\|x_{1}-x_{*}\|+\gamma_{k},\quad \text{for all}\ k\geq k_{t-1}\enspace,
\end{align}
where the last inequality follows from (\ref{eq:dist-opt}) and the sub-additivity of square-root.

{\bf Controlling the number of phases.} The bound (\ref{eq:dist-init}) plays a central role in our step-size tuning. Indeed, we observe that if $\|x^+_{t}(k_{t-1})-x_{1}\| > 3 \gamma_{k_{t-1}}$, then it means that $\gamma_{k_{t-1}}<\|x_{1}-x_{*}\|$, and our step-size tuning then increases $k$ until the condition $\|x^+_{t}(k)-x_{1}\| \leq 3 \gamma_{k}$ is met.
In addition, we check below that the design of $B^{\text{simple}}_{t+1}(k)$ ensures that  we have $k_{t}\leq k^*$ for all $t\leq T$, where
 $k^*\geq 1$ is the integer defined by $\gamma_{k^*-1} \leq D_{\gamma_{0}}:=\|x_{*}-x_{1}\|\vee \gamma_{0}< \gamma_{k^*}$,
and fulfilling
 \begin{equation}\label{leq:kstar}
 k^* \leq 1+\log_{2} \pa{ {\|x_{*}-x_{1}\|\over \gamma_{0}}\vee 1}=\log_{2} \pa{ {2D_{\gamma_{0}}\over \gamma_{0}}},\quad \text{and}\quad \gamma_{k^*}\leq 2 D_{\gamma_{0}}\enspace.
 \end{equation}
 Indeed, if $k_{t-1}\leq k^*$, then (\ref{eq:dist-init}) ensures that
\[\|x^+_{t}(k^*)-x_{1}\|\leq 2 \gamma_{k^*}+\gamma_{k^*}=B^{\text{simple}}_{t+1}(k^*)\,\,,\]
so $k_{t}\leq k^*$, and by induction the property holds for all $t\leq T$.

{\bf Bounding the regret.}
Let us now upper-bound the regret. We denote by $[T_{k},T_{k+1}-1]$ the interval where $k_t=k$, with the convention $T_{k+1}=T_{k}$ if we never have $k_t=k$, see Figure~\ref{fig:schematic} for schematic illustration.
Summing the central equation (\ref{eq:one-step}), the regret can then be decomposed as follows
\begin{align}\label{leq:phase-sum}
R_{T}&= \sum_{k=1}^{k_{T}}\sum_{t=T_{k}}^{T_{k+1}-1} (f(x_{t})-f(x_{*}))\nonumber\\
& \leq  \sum_{k=1}^{k^*} \pa{{\gamma_{k}L\over 2 \sqrt{T}}  (T_{k+1}-T_{k}) + {L\sqrt{T}\over 2 \gamma_{k}} \pa{\|x_{T_{k}}-x_{*}\|^2-\|x_{T_{k+1}}-x_{*}\|^2}}\nonumber\\
& \leq {L\sqrt{T}\over 2 }   \pa{\gamma_{k^*} + \sum_{k=1}^{k^*}{1\over  \gamma_{k}} \pa{\|x_{T_{k}}-x_{*}\|^2-\|x_{T_{k+1}}-x_{*}\|^2}}.
\end{align}
From the step-size rule, we have that $\|x_{T_{k+1}}-x_{1}\|\leq B_{T_{k+1}}(k_{T_{k+1}-1})=3\gamma_{k}$, and from (\ref{eq:dist-opt})  we have
$\|x_{T_{k}}-x_{*}\|^2\leq  \|x_{1}-x_{*}\|^2 +{\gamma_{k-1}^2}$, so we can upper bound the last term in the right-hand side of (\ref{leq:phase-sum})
\begin{align}
\|x_{T_{k}}-x_{*}\|^2-\|x_{T_{k+1}}-x_{*}\|^2 &\leq \|x_{1}-x_{*}\|^2 +{\gamma_{k-1}^2} -  \cro{\|x_{1}-x_{*}\|- \|x_{T_{k+1}}-x_{1}\|}^2_{+}\nonumber\\
&\leq \gamma_{k-1}^2+ \|x_{1}-x_{*}\|^2 - \cro{\|x_{1}-x_{*}\|- 3\gamma_{k}}^2_{+}\nonumber\\
&\leq {1\over 4}\gamma_{k}^2+6 \gamma_{k} \|x_{1}-x_{*}\|\enspace, \label{leq:nice}
\end{align}
where the last inequality follows from the basic inequality
$\Delta^2-[\Delta-B]_{+}^2\leq 2\Delta B$, for all $\Delta,B\geq 0$.
Substituting (\ref{leq:nice}) in (\ref{leq:phase-sum}) and
using the bound (\ref{leq:kstar}),
we end with the upper-bound
\begin{align}
R_{T}&\leq {L\sqrt{T}\over 2} \cro{\gamma_{k^*}{+}{\gamma_{k^*+1}\over 4} + 6k^*\|x_{1}-x_{*}\| }
\leq L\sqrt{T} \cro{3\|x_{1}-x_{*}\| \log_{2}\pa{ {2D_{\gamma_{0}}\over \gamma_{0}}} {+} 2 D_{\gamma_{0}}}\,.\label{bound:regret:simple}
\end{align}
The bound (\ref{bound:regret:simple}) for Algorithm~\ref{algo:PGDloglogMath}, then matches the optimal rate $\|x_{1}-x_{*}\| L\sqrt{T}$ obtained with the oracle step size $\eta=\|x_{1}-x_{*}\|/( L\sqrt{T})$, up to a factor  $\log_{2}\pa{ D_{\gamma_{0}}/ \gamma_{0}}$.

It turns out that, in this $L$-Lipschitz setting, it is possible to adapt to $\|x_{1}-x_{*}\|$ with
a bound $\cO\big({D_{\gamma_{0}}L\sqrt{T \log_{2}\pa{ D_{\gamma_{0}}/ \gamma_{0}}}}\big)$  on the regret, by, for example, using coin betting~\cite{pmlr-v35-mcmahan14,orabona2016coin}.
We achieve such tighter bound with PGD with a better tuning of the threshold $B_{t+1}(k)$ which is explained in the next section.

\subsection{Improving log factor by better tuning of $B_{t+1}(k)$}

Previous section gave the basic intuition, explaining why such a doubling strategy works. Yet, our choice of $B_{t+1}$ on Line \ref{alg:B} of Algorithm~\ref{algo:PGDloglogMath} differs from $B^{\text{simple}}_{t+1}(k) = 3\gamma_k$. Remaining in the simple setup of warmup, let us explain two key ingredients, which eventually lead to our choice of $B_{t+1}$ on Line \ref{alg:B} of Algorithm~\ref{algo:PGDloglogMath}.
The first ingredient is to track more tightly the upper-bound on $\|x^+_{t}(k)-x_{1}\|$. Indeed, we can improve
the Bound (\ref{eq:dist-opt}) by keeping
$\|x^+_{t}(k)-x_{*}\|^2\leq \|x_{1}(k)-x_{1}\|^2+ \Gamma^2_{t}+\gamma^2_{k}/T$ instead of relying on the last bound in (\ref{eq:dist-opt}).
Hence, we can replace (\ref{eq:dist-init}) by
\begin{equation}\label{better-track}
\|x^+_{t}(k)-x_{1}\| \leq 2 \|x_{1}-x_{*}\| +\sqrt{\Gamma^2_{t}+\gamma^2_{k}/T}\enspace,
\end{equation}
in order to implicitly track the value $ \|x_{1}-x_{*}\|$.
This improved tracking alone is not enough in order to improve the $\log$ factor. Indeed, choosing $B_{t+1}(k)={2 \gamma_{k}}+\sqrt{\Gamma^2_{t}+{\gamma_{k}^2/ T}}$, still introduces $\log(D_{\gamma_0})$ term.
To improve the $\log$ factor, our second ingredient is to choose a slightly smaller threshold $B_{t+1}$, at the price of possibly moderately  increasing the number $k_{T}$ of doubling. In particular, setting
\begin{align}
    \label{eq:B_less_simple}
    B_{t+1}(k)={2 \gamma_{k}\over \sqrt{k}}+\sqrt{\Gamma^2_{t}+{\gamma_{k}^2/ T}}\ ,
\end{align}
we get that $k_{T}\leq k^*+ 0.5 \log_{2}(k^*)+1.25$, and instead of $\log(D_{\gamma_0})$ we have $\sqrt{\log(D_{\gamma_0})}$ in the regret bound (see Appendix~\ref{app:warm-up} for details).
Combining everything together, we get the bound
\[
R_{T}\leq 10 D_{\gamma_{0}} {L\sqrt{T}}
  \sqrt{2\log_{2} \pa{ {2D_{\gamma_{0}}/ \gamma_{0}}}}\enspace,\]
for algorithm in~\eqref{warm-choice} with $3\gamma_k$ replaced by $B_{t+1}(k)$ in~\eqref{eq:B_less_simple} (see Theorem~\ref{thm:warm-up} in  Appendix~\ref{app:warm-up}).
While, the above discussion was still assuming the simple setup of $L$-Lipschitz function $f$, known $L$ and $T$, we are able to generalize the above argument to nearly arbitrary convex $f$ and unknown $T$.

\section{Meta theorem: a general case of Algorithm~\ref{algo:PGDloglogMath}}
\label{sec:meta}













In this section, we provide a unified analysis of Algorithm~\ref{algo:PGDloglogMath}, that is valid under the minimal Assumption~\ref{ass:main}, and for the choice of \emph{any positive non-decreasing sequence} $(h_t)_{t \geq 1}$ on Line~\ref{alg:line_h} of Algorithm~\ref{algo:PGDloglogMath}.
Our main result, stated in Theorem~\ref{thm:main_intro}, is obtained as a consequence of this general result, and is made precise in Corollary~\ref{cor:main}.
\begin{theorem}
    \label{thm:any_h}
  Let Assumption~\ref{ass:main} be satisfied.
  For any $\gamma_0 > 0$, for any positive non-decreasing $(h_t)_{t \geq 1}$ on Line~\ref{alg:line_h} of Algorithm~\ref{algo:PGDloglogMath}, Algorithm~\ref{algo:PGDloglogMath} satisfies
    \begin{align*}
    \sum_{t = 1}^T(f(x_t) - f(x_*))
    &\leq
    h_{T+1}\pa{2\norm{x_1 - x_*}\sqrt{k_T}\pa{2 {+} \sqrt{\frac{1}{3}\sum_{t = 1}^{T}\frac{\|g_t\|^2}{h^2_t}}} + \gamma_{k_T}\sum_{t = 1}^{T}\frac{\|g_t\|^2}{h^2_t}}\enspace.
  \end{align*}
\end{theorem}
It is interesting to observe that the term $\sum_{t \leq T} {\|g_t\|^2}/{h_t}$, that often appears in the analysis of \AdaGrad is absent in our bound. Instead, we have $\sum_{t \leq T} {\|g_t\|^2}/{h_t^2}$ which behaves slightly worse and hence requires additional correction of $h_t$ (extra $\log$ factor) to ensure convergence.
The proof of Theorem~\ref{thm:any_h}, which can be found in Appendix~\ref{app:meta},  is based on the following general lemma.
\begin{lemma}
\label{lem:ada_main}
Let Assumption~\ref{ass:main} be satisfied. Consider the following algorithm for $t \geq 1$
\begin{align*}
    x_{t + 1} = \Proj_{\Theta}\pa{x_t - \frac{\gamma}{{h_t}} g_t}\enspace,
\end{align*}
where $g_t \in \partial f(x_t)$, $(h_t)_{t \geq 1}$ is non-decreasing and positive, and $x_{1} \in \bbR^d$.
For all $T > 1$, and all $x_1 \in \bbR^d$, we have
    \begin{align*}
        \sum_{t = 1}^{T}(f(x_t) - f(x_*)) \leq
        {h_{T+1}}\pa{{\frac{\norm{x_{1} - x_*}^2 - \norm{x_{T+1} - x_*}^2}{2\gamma}}
        + \frac{\gamma}{2}\sum_{t = 1}^{T}\frac{\norm{g_t}^2}{h^2_t}}\enspace.
    \end{align*}
\end{lemma}
The above lemma replaces the key inequality~\eqref{eq:one-step} that was available for one step of PGD in the simplest case.
However, since the step-size in our case is time-varying, we rather need a variant of this inequality over the whole trajectory. While simple to prove, it seems that this result is novel and could be of independent interest.

Finally, to obtain the bound of Theorem~\ref{thm:main_intro}, we only need to bound the number of phases $k_T$. Note that the intuition of the previous section still applies in this case, yet, the actual bound on $k_T$ is more refined---it gives better constants and improves logarithmic factors.

\begin{lemma}
    \label{lem:number_of_phases}
    Let Assumption~\ref{ass:main} be satisfied.
  For any $\gamma_0 > 0$, and any non-decreasing positive $(h_t)_{t \geq 1}$, Algorithm~\ref{algo:PGDloglogMath} satisfies for $T \geq 2$
  \begin{align*}
    k_T \leq k^* + \frac{1}{2}\log_2(k^*) + \frac{5}{4}\qquad\text{and}\qquad \gamma_{k_T} \leq \frac{5}{2}\sqrt{k^*}\pa{\gamma_0 2^{k^*}}\enspace.
  \end{align*}
  where $k^*$ is such that $\gamma_0 2^{k^* - 1} \leq \|x_1 - x_*\| \vee \gamma_0 \leq \gamma_0 2^{k^*}$. Furthermore, $k_T = 1$ if $k^*=1$.
\end{lemma}
As a direct consequence of the above lemma and recalling that $D_{\gamma_0} = \|x_1 - x_*\|\vee\gamma_0$, we obtain
\begin{align}
    \label{eq:bound_on_k_T}
    \sqrt{k_T} \leq \sqrt{2}\sqrt{\log_2\pa{\frac{D_{\gamma_0}}{\gamma_0}} + 1} \qquad\text{and}\qquad \gamma_{k_T} \leq 5D_{\gamma_0}\sqrt{\log_2\pa{\frac{D_{\gamma_0}}{\gamma_0}} + 1}\enspace,
\end{align}
that is, $k_t$ takes at most $2 (\log_2(D_{\gamma_0} / \gamma_0) + 1)$ values.

\begin{proof}[Proof of Lemma~\ref{lem:number_of_phases}]
Lemma~\ref{lem:deviation} in Appendix, applied by phases, implies that for all $t \geq 1$ and $k \geq 1$ we have
\begin{align*}
    \norm{x_t^+(k) - x^*}^2 \leq \norm{x_1 - x_*}^2 + \Gamma_t^2 + \gamma_k^2\frac{\|g_t\|^2}{h^2_t}\enspace.
\end{align*}
Thus, the triangle inequality, yields
\begin{align*}
    \norm{x_t^+(k) - x_1}
    &\leq 2\norm{x_1 - x_*} + \sqrt{\Gamma_t^2 + \gamma_k^2\frac{\|g_t\|^2}{h^2_t}} 
    \leq
    2D_{\gamma_0} + \sqrt{\Gamma_t^2 + \gamma_k^2\frac{\|g_t\|^2}{h^2_t}}  
    \enspace,
\end{align*}
where $D_{\gamma_0} = \norm{x_1 - x_*} \vee \gamma_0$. Let $\bar{k}$ be the smallest integer such that ${2^{\bar{k}}}/{\sqrt{\bar k}} \geq 2^{k^*}$. Then, for any $k \geq 1$ and any $t\geq 1$
\begin{align*}
    \norm{x_t^+(k) - x_1} \leq \frac{2\gamma_{\bar k}}{\sqrt{\bar k}} + \sqrt{\Gamma_t^2 + \gamma_k^2\frac{\|g_t\|^2}{h^2_t}} \enspace. 
\end{align*}
In particular, the above implies that $\norm{x_t^+(\bar{k}) - x_1} \leq  B_{t+1}(\bar k)$ for all $t\geq 1$.
Thus, once $k_t$ reaches $\bar k$ on Line~\ref{alg:line_bound} of Algorithm~\ref{algo:PGDloglogMath}, it never changes its value. That is, $k_T \leq \bar{k}$. Lemma~\ref{lem:bark} in Appendix shows that $\bar k \leq k^* + 0.5\log_2(k^*) + 1.25$ and $\bar{k} = 1$ if $k^*=1$, which concludes the proof.
\end{proof}
\subsection{Applications of Theorem~\ref{thm:any_h}: specific choices of $(h_t)_{t \geq 1}$}
Theorem~\ref{thm:any_h} and Lemma~\ref{lem:number_of_phases} yield the main result of this work---theorem announced in Section~\ref{sec:main_result}.
\begin{corollary}
\label{cor:main}
    Under assumptions of Theorem~\ref{thm:any_h}. Let $H(x) = \sqrt{(x+1)\log(\e(x+1))}$.
    Setting $h_t = H(S_t)$ and $D_{\gamma_0} = \|x_1 - x_*\|\vee\gamma_0$, Algorithm~\ref{algo:PGDloglogMath} satisfies
  \begin{align*}
        \sum_{t = 1}^T(f(x_t) - f(x_*))
        \leq
        D_{\gamma_0}H(S_{T+1})\sqrt{\log_2\pa{\frac{2D_{\gamma_0}}{\gamma_0}}}
        \Bigg[\,
        6\log(\log(\e(1 + S_T))) + 6.5
        \, \Bigg]\enspace.
  \end{align*}
\end{corollary}
While the above choice of $(h_t)_{t \geq 1}$ gives nearly optimal rates, it is not standard in the literature. Let us highlight the usefulness of Theorem~\ref{thm:any_h} by providing some instantiations which correspond to other, more common, but less optimal, examples.


The standard \AdaGrad corresponds to $h_t = \sqrt{S_t}$~\cite{streeter2010less}. The main inconvenience of this choice, is that the term $\sum_{t \leq T} {\|g_t\|^2}/{S_t}$ is not bounded uniformly by a non-decreasing function of $S_T$.
Indeed, assume that $\|g_t\|^2 = 1 / T$ for all $t = 1, \ldots, T$, then $S_t = t / T \leq 1$ and
    $\sum_{t \leq T} {\|g_t\|^2}/{S_t} \approx \log(T)$.
It is possible, however, to write $\sum_{t \leq T} {\|g_t\|^2}/{S_t} \leq 1 + \log(S_T / \|g_1\|^2)$, which involves additional dependency on the gradient at initialization. All in all, we can state the following corollary.
\begin{corollary}
    \label{cor:adagrad_real}
    Under assumptions of Theorem~\ref{thm:any_h}.
    Setting $h_t = \sqrt{S_t}$ and $D_{\gamma_0} = \|x_1 - x_*\|\vee\gamma_0$, Algorithm~\ref{algo:PGDloglogMath} satisfies
    \begin{align*}
    \sum_{t = 1}^T(f(x_t) - f(x_*)) \leq D_{\gamma_0}\sqrt{S_{T+1}\log_2\pa{\frac{2D_{\gamma_0}}{\gamma_0}}}\
    \Bigg[6\log\pa{\frac{\e S_T}{\|g_1\|^2}} + 6.5\Bigg]\enspace.
  \end{align*}
\end{corollary}
An attractive feature of this bound is its scale-invariance---multiplying $f$ by some constant, multiplies the bound by the same constant.

The dependency on the initial gradient can be avoided setting $h_t = \sqrt{\varepsilon + S_t}$ with arbitrary $\varepsilon > 0$, as it is usually done in practice with \AdaGrad, and initially proposed in~\cite{duchi2011adaptive}.
\begin{corollary}
\label{cor:adagrad_like}
Under assumptions of Theorem~\ref{thm:any_h}. Let $h_t  = \sqrt{\varepsilon + S_{t}}$\,, for some $\varepsilon > 0$.
    Setting  $D_{\gamma_0} = \|x_1 - x_*\|\vee\gamma_0$, Algorithm~\ref{algo:PGDloglogMath} satisfies
    \begin{align*}
    \sum_{t = 1}^T(f(x_t) - f(x_*)) \leq D_{\gamma_0}\sqrt{\pa{S_{T+1}+\varepsilon}\log_2\pa{\frac{2D_{\gamma_0}}{\gamma_0}}}\
    \Bigg[6\log\pa{1 + \frac{S_T}{\varepsilon}} + 6.5\Bigg]\enspace.
  \end{align*}
\end{corollary}
Note that compared to Corollary~\ref{cor:main}, the above bound contains an additional $\sqrt{\log(1 + S_T)}$ multiplicative factor, but it improves upon that of Corollary~\ref{cor:adagrad_real}. Finally, we can also recover the results claimed in the end of Section~\ref{sec:warm-up}, where $f$ is assumed to be $L$-Lipschitz, see Appendix~\ref{app:cor:proof} for details.














\begin{algorithm}[t!]
\DontPrintSemicolon
\caption{Stochastic case}\label{algo:stoch}
\SetKwInput{Input}{Input}
   \SetKwInOut{Output}{Output}
   \SetKwInput{Initialization}{Initialization}
   \Input{$x_1 \in \bbR^d, \Theta \subset \bbR^d, \gamma_0 > 0, L > 0, T, \delta > 0$}
   \Initialization{$\Gamma^2_1 = 0, k_0 = 1, S_0 = 0, h : \bbR \to \bbR, \gamma_k = \gamma_0 2^k \text{ for $k \geq 1$}, \ell_{T}(\delta):= 1\vee \log(\log_{2}(2T)/\delta)$}

\For{$t = 1, \ldots, T$}{

    $g_t \in \partial F(x_t, \xi_t)$  \tcp*[r]{{\small get subgradient}}

    $h_{t}=L\sqrt{T \ell_{T}(\delta/(1+k_{t-1})^2)}$  \tcp*[r]{{\small update $h_{t}$}}

    $x^+_{t}(k) = \Proj_{\Theta}\pa{x_t - \frac{\gamma_k}{h_t}g_t}$ \tcp*[r]{\small probing step}

    $k_t = \min\left\{k \geq k_{t-1} \,:\, \|x^+_{t}(k) - x_1\| \leq 38 \gamma_k\right\}$ \tcp*[r]{\small find the step size}


    $x_{t + 1} = x^+(k_t)$ \tcp*[r]{\small make the step
    }

    }

    \Output{Trajectory $(x_t)_{t = 1}^T$}

\end{algorithm}

\section{An extension to stochastic optimization}
\label{sec:SGD}

In this section we demonstrate that at least the warm-up analysis provided in Section~\ref{sec:warm-up} extends to the setup of stochastic optimization (see Algorithm~\ref{algo:stoch}), where the objective function takes the form
\begin{align*}
    f(x) = \E[F(x, \xi)]\enspace,
\end{align*}
and where we only have access to  $g_t \in \partial F(x_t, \xi_t)$, for some i.i.d. $(\xi_{t})_{t \geq 1}$.
As in~\cite{carmon2023making}, we make the following standard assumption on the regularity of $F(\cdot, \xi)$.
\begin{assumption}
    \label{ass:stoch}
    The mapping $x \mapsto F(x, \xi)$ is $L$-Lipschitz almost surely.
\end{assumption}
In some applications (e.g., linear contextual bandits), $L$ is actually known and the control of regret is necessary. For example, in linear Contextual Bandits with Knapsacks (lin-CBwK), having PGD strategy for unbounded $\Theta$, while still controlling the regret is needed~\cite[see e.g.,][]{agrawal2016linear}. Thus, our Algorithm~\ref{algo:stoch}, could bring new results in CBwK and related contexts.

We can state the following result concerning Algorithm~\ref{algo:stoch}.

\begin{theorem}
\label{thm:SGD}
Let Assumptions~\ref{ass:main} and~\ref{ass:stoch} be satisfied.
Define $\ell_T(\delta) = 1 \vee \log(\log_2(2T) / \delta)$ and $D_{\gamma_0} = \|x_1 - x_*\|\vee\gamma_0$.
For any $\gamma_{0} > 0$,
Algorithm~\ref{algo:stoch} satisfies with probability at least $1 - \delta$
  \begin{align*}
    \sum_{t = 1}^T(f(x_t) - f(x_*))
     &\leq
     3500D_{\gamma_0}L\sqrt{T}\log_2\pa{\frac{2D_{\gamma_0}}{\gamma_0}}\ell_T^{\sfrac12}\big({\sfrac{\delta}{\log_2^2(4D_{\gamma_0} / \gamma_0)}}\big)\enspace.
\end{align*}
\end{theorem}
The above bound is of order $\cO(LD_{\gamma_0}\sqrt{T}\log(D_{\gamma_0})\log\log(TD_{\gamma_0}))$.
Note that if only the optimization error is of concern, and one does not wish to control the regret,~\cite{carmon2023making} provide a bound without $\log_2\pa{{2D_{\gamma_0}}/{\gamma_0}}$ using bisection algorithm and several restarted runs of SGD.


\begin{figure}[t!]
 \centering
   \includegraphics[width=0.49\textwidth]{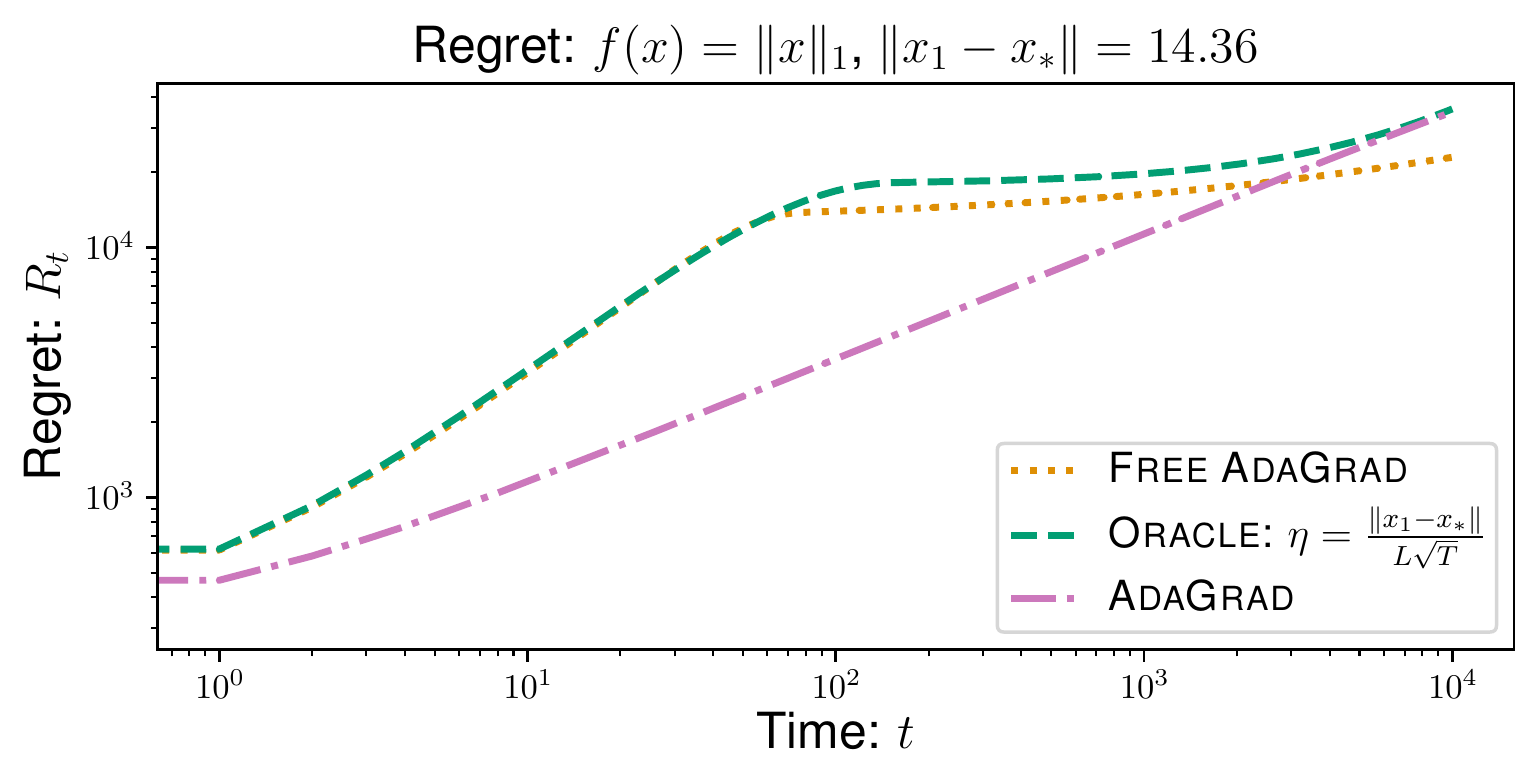}
   \hfill
   \includegraphics[width=0.49\textwidth]{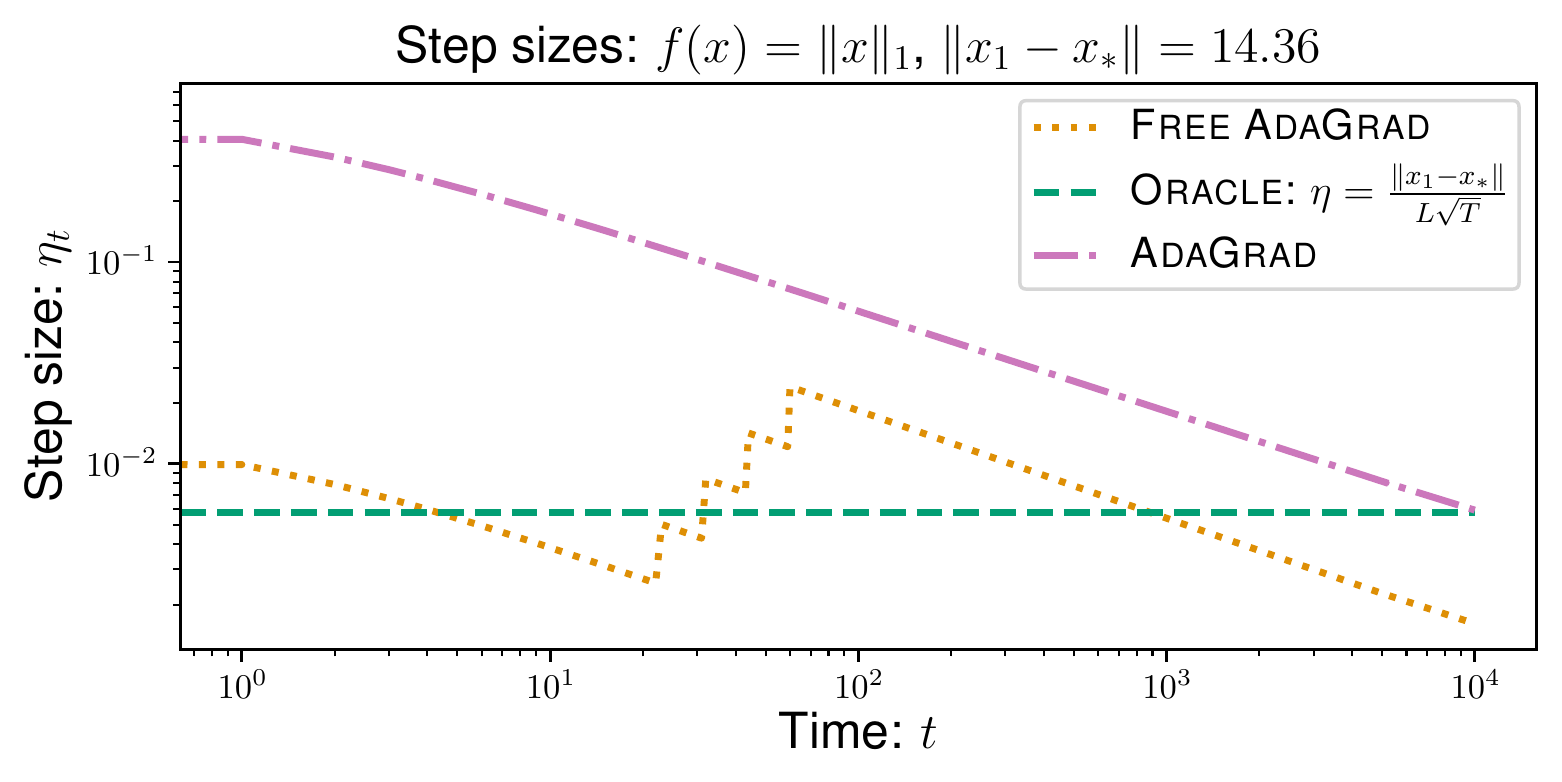}

   \includegraphics[width=0.49\textwidth]{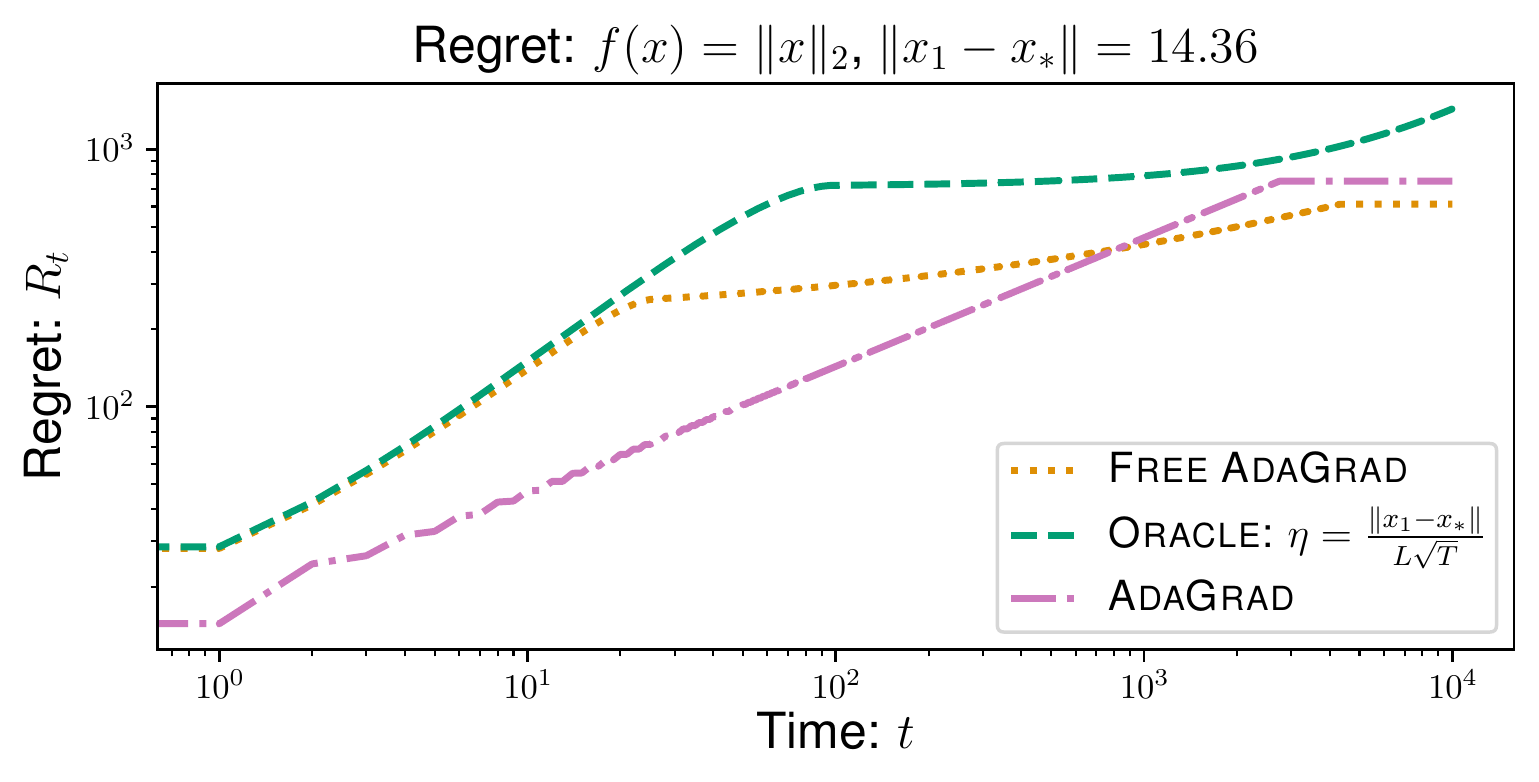}
   \hfill
   \includegraphics[width=0.49\textwidth]{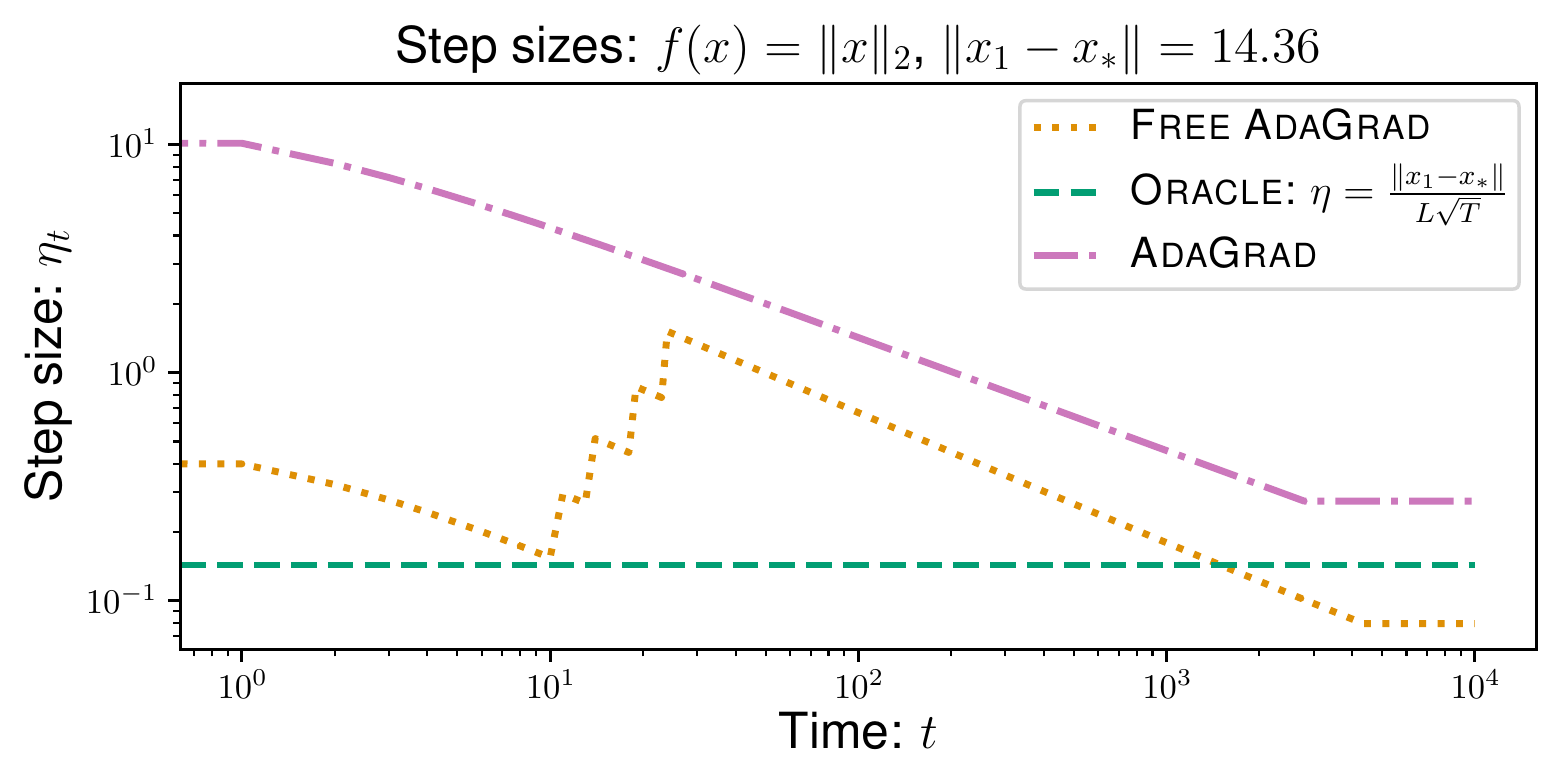}

    \includegraphics[width=0.49\textwidth]{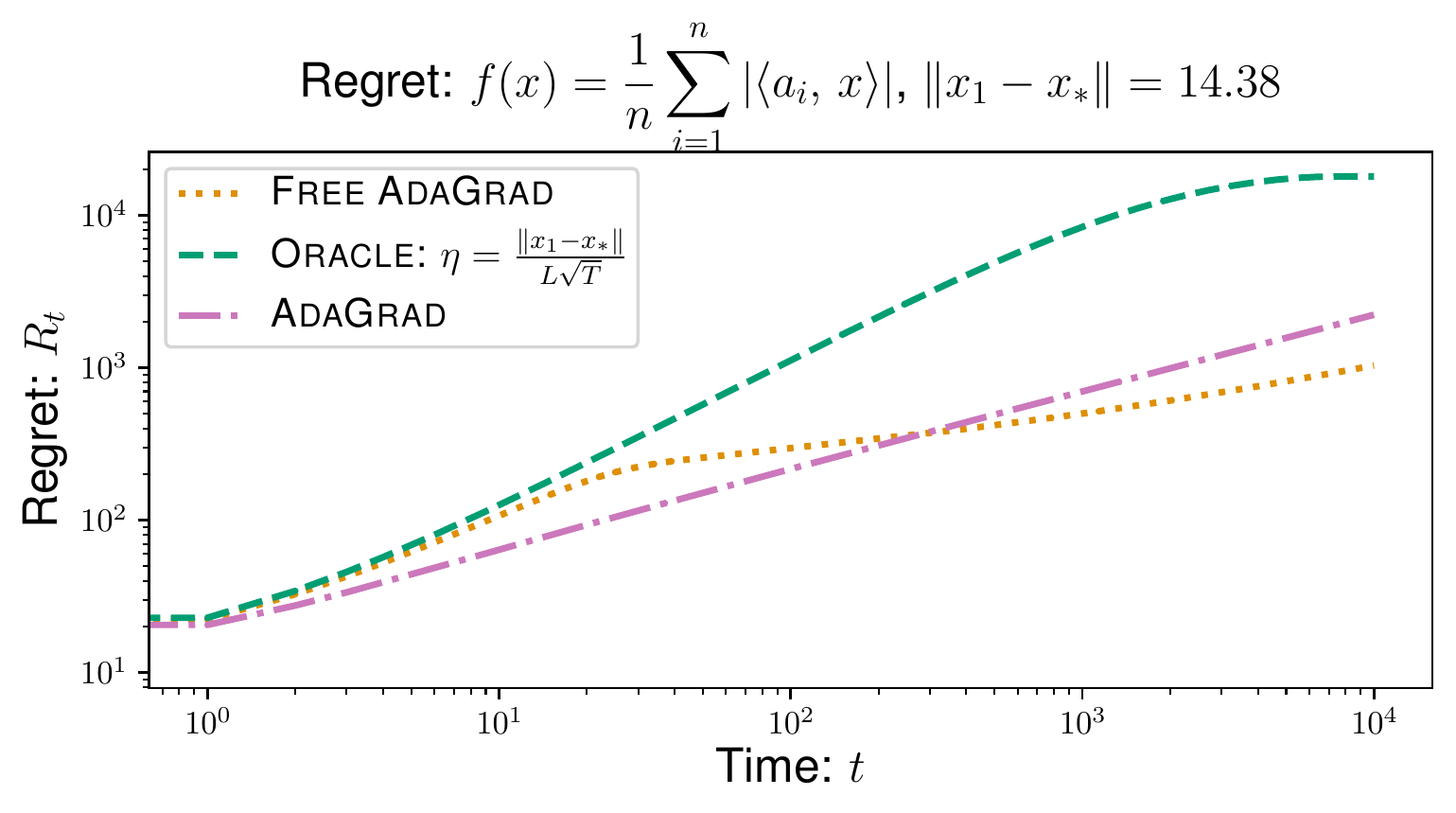}
   \hfill
   \includegraphics[width=0.49\textwidth]{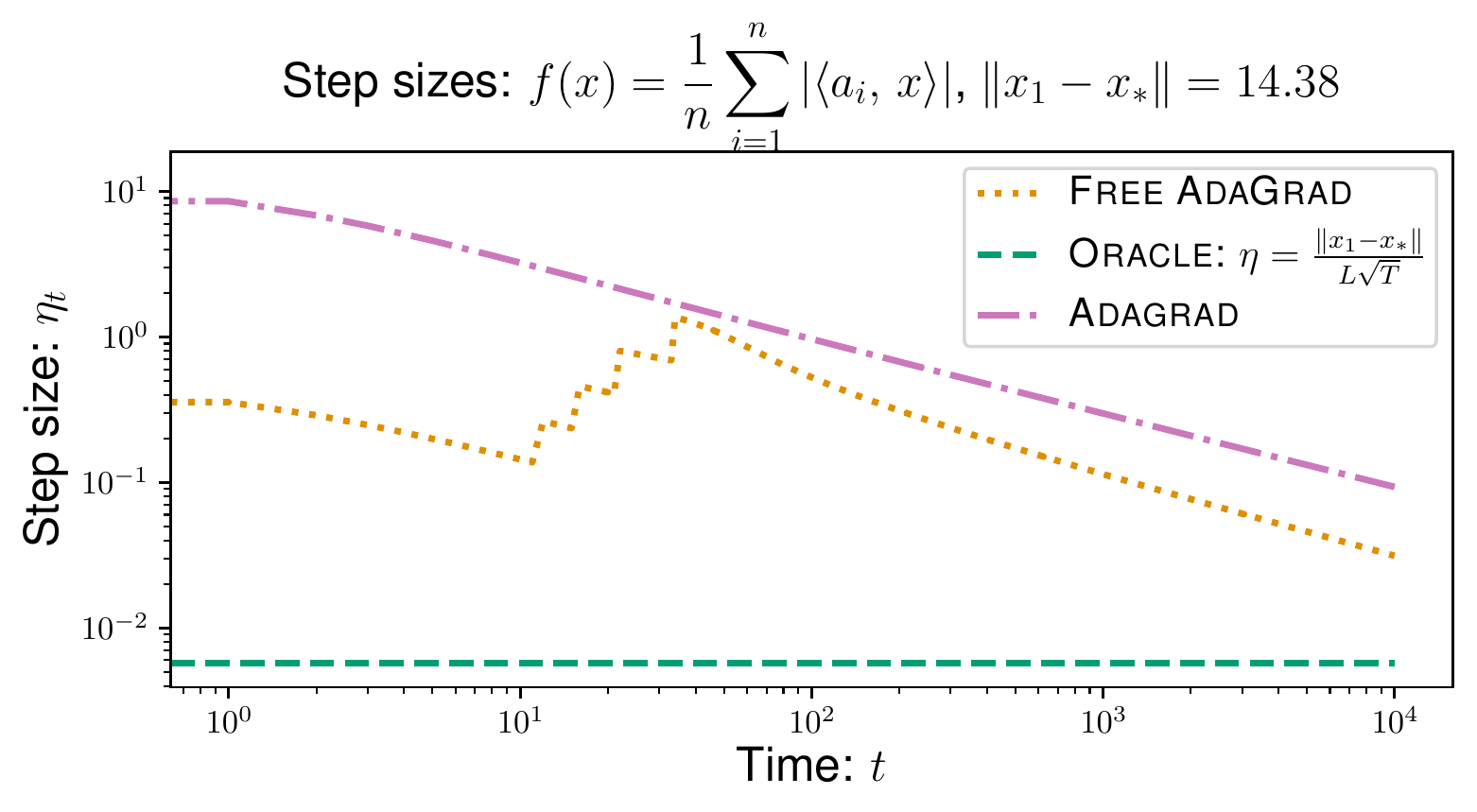}
   \caption{Regret (left) and step-sizes (right) of three algorithms on $\log{-}\log$ scale.}\label{fig:simple}
\end{figure}




\section{Experiments}
\label{sec:exps}

We have implemented our \FreeAdaGrad (with $\gamma_0 = 1$ throughout) algorithm and compared it to the \AdaGrad that requires the  knowledge of $\|x_1 - x_*\|$ and to the Oracle choice of step  $\|x_1 - x_*\|/ (L \sqrt{T})$.\\
We consider three functions $f(x) = \|x\|_p$ for $p \in \{1, 2\}$ and $f(x) = n^{-1}\sum_{i \leq n} |\scalar{a_i}{x}|$
where $a_i \in \bbR^d$ are generated i.i.d. from standard multivariate Gaussian. The initialization point is picked the same for the three algorithms and is sampled from uniform distribution on $[-1, 1]^d$. For our experiments, we set $d = 625$ and $n = 1000$. Note that in the first case, the considered function is Lipschitz with $L=1$ and for the second one $L \leq \tfrac{1}{n}(\|a_1\| + \ldots + \|a_n\|)$. A subgradient at $x \in \bbR^d$ in the second case is given by $n^{-1}\sum_{i \leq n} a_i \text{sign}(\scalar{a_i}{x})$
and since $a_i$'s are i.i.d. Gaussian, it is expected that $\|g\| \ll \tfrac{1}{n}(\|a_1\| + \ldots + \|a_n\|)$---algorithms that are adaptive to the norm of the gradient should perform better in this case. For all three functions, a global minimizer is given by $x_* = (0, \ldots, 0)^\top$. All the algorithms run for $T = 10000$ iterations.\\
All the plots are reported on $\log{-}\log$ scale.
The first results are reported on Figure~\ref{fig:simple}. The second column displays the step-sizes used by the three algorithms. As a sanity check, we observe that the step size of \AdaGrad decreases over time and the step size of the \Oracle remains constant. One can also observe the characteristic jumps of the proposed \FreeAdaGrad method---the step size decreases withing a fixed phase and is doubled from one phase to the other. On the first row of Figure~\ref{fig:simple} we display the regret, on initial stages our algorithm behaves similarly to the \Oracle one, while surpassing the performance of the \AdaGrad on the later stages. The third row of Figure~\ref{fig:simple} displays the case of the averages. Note that in this case the \Oracle algorithm performs worse than the other two, since it takes the worst-case Lipschitz constant and does not adapt to the actual norms of the seen gradients.

\section{Discussion}
We have introduced \FreeAdaGrad---a simple fully adaptive version of \AdaGrad, that does not rely on any prior information about the objective function.
Our bounds are optimal up to logarithmic factors and are applicable to non-globally Lipschitz functions.
We have extended our approach to stochastic optimization in a Lipschitz context, at the cost of the knowledge of the Lipschitz constant and sub-optimal logarithmic factors.
Numerical illustrations suggest that \FreeAdaGrad performs on par or outperforms \AdaGrad with knowledge of $\|x_{1}-x_{*}\|$ and the \Oracle choice of step-size.
\textbf{Limitations.} Let us also list the main limitations and future directions of our work.\smallskip\\
\textbf{1)} We are only dealing with batch optimization. The extension of our analysis to the case of Online Convex Optimization (OCO) seems non-trivial, since the bounds that we obtain are known to be unachievable without prior knowledge in the OCO context~\cite{cutkosky2017online}. The investigation of \FreeAdaGrad in the OCO setting is left for future work;\smallskip\\
\textbf{2)} If $f$ is assumed to be $L$-Lipschitz with known constant $L$, slightly better bounds---with improved log-factors---can be obtained in OCO setting~\cite[see e.g.,][]{orabona2016coin, cutkosky2018black}.
It remains an open question wether such bounds, can be obtained in batch optimization in the non-Lipschitz (or unknown $L$) and unknown $\|x_1 - x_*\|$ case;\smallskip\\
\textbf{3)} Concerning stochastic optimization, we require $f$ to be $L$-Lipschitz for some known $L$. We note, however, that even in the state-of-the-art bound of~\cite{carmon2023making}, the knowledge of $L$ is required.\smallskip\\
\textbf{4)} When translating our regret bound
on a rate
for the optimization error, using $\bar x_{T}$---average along the trajectory---we have additional log-factors compared to~\cite{carmon2023making, defazaio2023learning}, which is an artifact of online-to-batch conversion~\cite[Theorem 7]{mcmahan2012no}.
Contrary to us, the algorithms in~\cite{carmon2023making, defazaio2023learning} require yet some knowledge about the Lipschitz constant $L$.

{
}

\bibliography{biblio.bib}

\newcommand{\etalchar}[1]{$^{#1}$}
\begin{thebibliography}{CBLS05}

\bibitem[AD16]{agrawal2016linear}
Shipra Agrawal and Nikhil Devanur.
\newblock Linear contextual bandits with knapsacks.
\newblock {\em Advances in Neural Information Processing Systems}, 29, 2016.

\bibitem[CB17]{cutkosky2017online}
Ashok Cutkosky and Kwabena Boahen.
\newblock Online learning without prior information.
\newblock In {\em Conference on learning theory}, pages 643--677. PMLR, 2017.

\bibitem[CBLS05]{Bern}
N.~Cesa-Bianchi, G.~Lugosi, and G.~Stoltz.
\newblock Minimizing regret with label-efficient prediction.
\newblock {\em IEEE Transactions on Information Theory}, 51:2152--2162, 2005.

\bibitem[CH23]{carmon2023making}
Yair Carmon and Oliver Hinder.
\newblock Making sgd parameter-free, 2023.

\bibitem[CO18]{cutkosky2018black}
Ashok Cutkosky and Francesco Orabona.
\newblock Black-box reductions for parameter-free online learning in banach
  spaces.
\newblock In {\em Conference On Learning Theory}, pages 1493--1529. PMLR, 2018.

\bibitem[Cut19]{pmlr-v99-cutkosky19a}
Ashok Cutkosky.
\newblock Artificial constraints and hints for unbounded online learning.
\newblock In Alina Beygelzimer and Daniel Hsu, editors, {\em Proceedings of the
  Thirty-Second Conference on Learning Theory}, volume~99 of {\em Proceedings
  of Machine Learning Research}, pages 874--894. PMLR, 25--28 Jun 2019.

\bibitem[DHS11]{duchi2011adaptive}
John Duchi, Elad Hazan, and Yoram Singer.
\newblock Adaptive subgradient methods for online learning and stochastic
  optimization.
\newblock {\em Journal of machine learning research}, 12(7), 2011.

\bibitem[DM23]{defazaio2023learning}
Aaron Defazio and Konstantin Mishchenko.
\newblock Learning-rate-free learning by {D}-adaptation.
\newblock {\em arXiv preprint arXiv:2301.07733}, 2023.

\bibitem[JC22]{jacobsen2022parameter}
Andrew Jacobsen and Ashok Cutkosky.
\newblock Parameter-free mirror descent.
\newblock In {\em Conference on Learning Theory}, pages 4160--4211. PMLR, 2022.

\bibitem[MK20]{pmlr-v125-mhammedi20a}
Zakaria Mhammedi and Wouter~M. Koolen.
\newblock Lipschitz and comparator-norm adaptivity in online learning.
\newblock In Jacob Abernethy and Shivani Agarwal, editors, {\em Proceedings of
  Thirty Third Conference on Learning Theory}, volume 125 of {\em Proceedings
  of Machine Learning Research}, pages 2858--2887. PMLR, 09--12 Jul 2020.

\bibitem[MO14a]{mcmahan2014unconstrained}
H~Brendan McMahan and Francesco Orabona.
\newblock Unconstrained online linear learning in hilbert spaces: Minimax
  algorithms and normal approximations.
\newblock In {\em Conference on Learning Theory}, pages 1020--1039. PMLR, 2014.

\bibitem[MO14b]{pmlr-v35-mcmahan14}
H.~Brendan McMahan and Francesco Orabona.
\newblock Unconstrained online linear learning in hilbert spaces: Minimax
  algorithms and normal approximations.
\newblock In Maria~Florina Balcan, Vitaly Feldman, and Csaba Szepesvári,
  editors, {\em Proceedings of The 27th Conference on Learning Theory},
  volume~35 of {\em Proceedings of Machine Learning Research}, pages
  1020--1039, Barcelona, Spain, 13--15 Jun 2014. PMLR.

\bibitem[MS12]{mcmahan2012no}
Brendan Mcmahan and Matthew Streeter.
\newblock No-regret algorithms for unconstrained online convex optimization.
\newblock {\em Advances in neural information processing systems}, 25, 2012.

\bibitem[N{\etalchar{+}}18]{nesterov2018lectures}
Yurii Nesterov et~al.
\newblock {\em Lectures on convex optimization}, volume 137.
\newblock Springer, 2018.

\bibitem[OP16]{orabona2016coin}
Francesco Orabona and D{\'a}vid P{\'a}l.
\newblock Coin betting and parameter-free online learning.
\newblock {\em Advances in Neural Information Processing Systems}, 29, 2016.

\bibitem[OP21]{orabona2021parameter}
Francesco Orabona and D{\'a}vid P{\'a}l.
\newblock Parameter-free stochastic optimization of variationally coherent
  functions.
\newblock {\em arXiv preprint arXiv:2102.00236}, 2021.

\bibitem[OT17]{orabona2017training}
Francesco Orabona and Tatiana Tommasi.
\newblock Training deep networks without learning rates through coin betting.
\newblock {\em Advances in Neural Information Processing Systems}, 30, 2017.

\bibitem[SM10]{streeter2010less}
Matthew Streeter and H~Brendan McMahan.
\newblock Less regret via online conditioning.
\newblock {\em arXiv preprint arXiv:1002.4862}, 2010.

\bibitem[ZCP22]{zhang2022pde}
Zhiyu Zhang, Ashok Cutkosky, and Ioannis Paschalidis.
\newblock Pde-based optimal strategy for unconstrained online learning.
\newblock In {\em International Conference on Machine Learning}, pages
  26085--26115. PMLR, 2022.

\end{thebibliography}
\bibliographystyle{alpha}

\newpage
\appendix

\begin{center}
    {\Large\bf Supplementary material for\\
    ``Parameter-free projected gradient descent’’}\\
    \vspace{1cm}
\end{center}
Appendix~\ref{app:warm-up} provides details for the proof of Section \ref{sec:warm-up} of the main body. Appendix~\ref{app:meta} contains all the proof for Theorem~\ref{thm:any_h} and Lemma~\ref{lem:ada_main}.
Appendix~\ref{app:cor:proof} provides proof for corollaries in Section~\ref{sec:meta}. Appendix~\ref{app:SGD} deals with stochastic version of our algorithm and contains the proof of Theorem~\ref{thm:SGD}. Appendix~\ref{app:MS12} gives detailed connection with the reward doubling algorithm of~\cite{mcmahan2012no}. Finally, Appendix~\ref{app:aux} contains auxiliary results that are used in different parts of the proofs.

Below we provide a basic \texttt{python} implementation of our \FreeAdaGrad.

\begin{python}
import numpy as np

class ObjectiveFunction():
    """
        Class for objective function has get_subgradient method
    """

    def get_subgradient(self, x):
        # To implement
        pass

def step(x, eta, g):
    return x - eta * g

def free_adagrad(stopping_criteria, obj_func, x1, gamma0=1.):
    """
        x1: initialization
        gamma0: initial guess for |x_1 - x_*|
        stopping_criteria: stopping criteria (e.g., max_iter)
        obj_func: objective function with get_subgrad() method
    """
    S = 0.
    Gamma = 0.
    k = 1
    gamma = gamma0

    x = np.copy(x1)
    trajectory = [x1]

    while not stopping_criteria:
        g = obj_func.get_subgrad(x)
        norm_g = np.linalg.norm(g)
        S += norm_g ** 2
        h = np.sqrt((S + 1.) * (1. + np.log(1. + S)))
        while True:
            x_plus = step(x, gamma / h, g)
            B = (2. / np.sqrt(k)) * gamma \
                + np.sqrt(Gamma + (gamma * norm_g / h) ** 2)
            if np.linalg.norm(x_plus - x1) > B:
                k += 1
                gamma *= 2
            else:
                Gamma += (gamma * (norm_g / h)) ** 2
                break
        x = x_plus
        trajectory.append(x.copy())

return trajectory
\end{python}

\newpage

\section{Proofs of the results of the warm-up Section \ref{sec:warm-up}}
\label{app:warm-up}


We prove here with full details the results of the warm-up Section \ref{sec:warm-up}, in the setting where the norm of the subgradients are bounded by a known constant $L$, and where the time horizon $T$ is known.
We set $h_t=L\sqrt{T}$, and we analyze simultaneously \FreeAdaGrad algorithm~\ref{algo:PGDloglogMath} with this choice of $h$, and the simple variant, where we  set $B^{\text{simple}}_{t+1}(k)=3\gamma_{k}$ for the threshold, as in Section \ref{sec:warm-up}.

\begin{theorem}\label{thm:warm-up}
Assume that $f$ is a convex $L$-Lipschitz function, such that there exists $x_* \in \argmin_{x \in \Theta} f(x)$ bounded. Let $\gamma_{0}>0$ and  $D_{\gamma_{0}}=\|x_{1}-x_{*}\|\vee \gamma_{0}$. The \FreeAdaGrad algorithm~\ref{algo:PGDloglogMath} with $h_t=L\sqrt{T}$ and $B_{t + 1}(k) = \gamma_k\pa{\frac{2}{\sqrt{k}} + \frac{1}{\sqrt{T}}} + \Gamma_t$ fulfills
\begin{align*}
\sum_{t=1}^T (f(x_{t})-f(x_{*}))&\leq 10 D_{\gamma_{0}} {L\sqrt{T}}
  \sqrt{2\log_{2} \pa{ {2D_{\gamma_{0}}\over \gamma_{0}}}}\,.
\end{align*}
The simple variant with $h_t=L\sqrt{T}$ and $B^{\text{simple}}_{t + 1}(k) =3\gamma_{k}$ fulfills
\[\sum_{t=1}^T (f(x_{t})-f(x_{*}))\leq 3\|x_{1}-x_{*}\| L\sqrt{T}\ \log_{2}\pa{ {2D_{\gamma_{0}}\over \gamma_{0}}} + 2 D_{\gamma_{0}}L\sqrt{T}\,\, .\]
\end{theorem}

{\bf Proof of Theorem \ref{thm:warm-up}.}

We start by emphasizing that the algorithm runs without diverging, in the sense that
\begin{equation}\label{eq:inflate}
k_{t}:=\min\ac{ k \geq k_{t-1}\ \textrm{such that}\ \|x^+_{t}(k) - x_{1} \| \leq B_{t+1}(k)}
\end{equation}
is finite for any $t$. Indeed, we observe that $ \|x^+_{t}(k) - x_{1} \|$ grows at most like $\gamma_{k} /\sqrt{T}$ when $k$ goes to infinity, while $B_{t+1}(k)$ grows faster than  $\gamma_{k}\pa{2/\sqrt{k}+1 /\sqrt{T}}$ and $B^{\text{simple}}_{t+1}(k)$ grows like $3 \gamma_{k}$. In fact, we will prove below that $k_{t}$ remains upper-bounded by a quantity independent of $T$.

The starting point of the proof is the classical analysis for a projected gradient step $\Proj_{\Theta}\pa{x_{t}-\eta g_{t}}$
\begin{align*}
f(x_{t})-f(x_{*}) &\leq \langle g_{t}, x_{t}-x_{*} \rangle
= {{\eta\over 2}\|g_{t}\|^2 +{1\over 2 \eta} \pa{\|x_{t}-x_{*}\|^2-\|x_{t}-\eta g_{t}-x_{*}\|^2}}\\
&\leq {{\eta\over 2}L^2 +{1\over 2 \eta} \pa{\|x_{t}-x_{*}\|^2-\|\Proj_{\Theta}\pa{x_{t}-\eta g_{t}}-x_{*}\|^2}}\,,
\end{align*}
where the last inequality follows from the fact that $x_{*}\in\Theta$.
Since $x_{*}$ is a minimizer of $f$ in $\Theta$, the left hand side is non-negative, so the above inequality with $\eta=\gamma_{k}/(L\sqrt{T})$ gives that
 for any $k\geq 1$
\begin{equation}\label{eq:one-step:app}
0\leq f(x_{t})-f(x_{*}) \leq \langle g_{t}, x_{t}-x_{*} \rangle\leq {{\gamma_{k}L\over 2\sqrt{T}} +{L\sqrt{T}\over 2 \gamma_{k}} \pa{\|x_{*}-x_{t}\|^2-\|x_{t}^+(k)-x_{*}\|^2}}\,.
\end{equation}
It follows from this bound, a one-step deviation upper-bound
\[\|x^+_{t}(k)-x_{*}\|^2\leq \|x_{t}-x_{*}\|^2+ {\gamma_{k}^2\over T}\,\,.\]
Summing this bound over $t$, we get a  bound on the distance to optimum
\begin{align}
\|x^+_{t}(k)-x_{*}\|^2&\leq \|x_{1}-x_{*}\|^2+\sum_{s=1}^{t-1} {\gamma^2_{k_{s}}\over T} +{\gamma_{k}^2\over T}=
 \|x_{1}-x_{*}\|^2 +\Gamma_{t}^2+{\gamma_{k}^2\over T}\,,\label{eq:dist-opt:app}
\end{align}
and then a bound on the distance to initialization
\begin{align}\label{eq:dist-init:app}
\|x^+_{t}(k)-x_{1}\|&\leq \|x_{1}-x_{*}\|+\|x^+_{t}(k)-x_{*}\| \nonumber\\
& \leq \|x_{1}-x_{*}\|+ \sqrt{\|x_{1}-x_{*}\|^2 + \Gamma_{t}^2 + {\gamma_{k}^2\over T}}\nonumber \\
&\leq 2\|x_{1}-x_{*}\|+\sqrt{\Gamma_{t}^2+{\gamma_{k}^2\over T}}\,, 
\end{align}
where the two inequalities follow from (\ref{eq:dist-opt:app}) and the sub-additivity of square-root.

{\bf Controlling the number $k_{T}$ of phases.}
The Inequality (\ref{eq:dist-init:app}) is the key to get an upper-bound on $k_{T}=\max\ac{k_{t}:1\leq t\leq T}$.
Let us  define the integer $k^*\geq 1$ by $\gamma_{k^*-1} \leq D_{\gamma_{0}}:=\|x_{*}-x_{1}\|\vee \gamma_{0}< \gamma_{k^*}$, which fulfills
\begin{equation}\label{leq:app:kstar}
k^* \leq 1+\log_{2} \pa{ {\|x_{*}-x_{1}\|\over \gamma_{0}}\vee 1}=\log\pa{2D_{\gamma_{0}}\over \gamma_{0}},\quad\text{and}\quad \gamma_{k^*}\leq 2 D_{\gamma_{0}}\,.
\end{equation}
To upper bound $k_{T}$, we rely on the following estimate derived from (\ref{eq:dist-init:app})
\begin{align}
\|x^+_{t}(k)-x_{1}\| 
&\leq 2\gamma_{k^*}+ \sqrt{\Gamma_{t}^2+ {\gamma^2_{k}\over T}}\,. \label{leq:break}
\end{align}
\begin{itemize}
\item {\bf Simple case: $B^{\text{simple}}_{t+1}(k)= 3\gamma_{k}$ .} A basic induction ensures that $k_t\leq k^*$ for all $t\leq T$. Indeed, if the property $k_{t-1}\leq k^*$ holds, then, since $\Gamma^2_{t} \leq  {t-1\over T} \gamma^2_{k_{t-1}}$,
 we have
\[\|x^+_{t}(k^*)-x_{1}\| \leq 2\gamma_{k^*}+\gamma_{k^*} = B^{\text{simple}}_{t+1}(k^*)\,\,,\]
which, in turn, ensures that $k_{t}\leq k^*$. So $k^*$ is an upper-bound on $k_{T}=\max\ac{k_{t}:1\leq t\leq T}$ in this case.
\item {\bf \FreeAdaGrad case: $B_{t+1}(k)= {2 \gamma_{\bar k} \over \bar k^{1/2}} + \sqrt{ \Gamma^2_{t}+ {\gamma^2_{\bar k}\over T}}$ .}
Let us define $\bar k$ as the smallest integer fulfilling $\bar k^{-1/2}\gamma_{\bar k}\geq \gamma_{k^*}$.
Then, from (\ref{leq:break}), we have
\[\|x^+_{t}(\bar k)-x_{1}\| \leq {2 \gamma_{\bar k} \over \bar k^{1/2}} + \sqrt{ \Gamma^2_{t}+ {\gamma^2_{\bar k}\over T}}=B_{t+1}(\bar k)\,\,,\]
so, by induction, we get $k_t\leq \bar k $ for all $t\leq T$.
In addition, we prove in Lemma \ref{lem:bark} page \pageref{lem:bark} that
\begin{equation}\label{borne:bark}
\bar k\leq k^*+0.5 \log_{2}(k^*)+1.25\,.
\end{equation}
\end{itemize}

{\bf Bounding the regret on phase $k_{t}=k$.}
We denote by $[T_{k},T_{k+1}-1]$ the interval where $k_t=k$, with the convention $T_{k+1}=T_{k}$ if we never have $k_t=k$. For $t\in [T_{k},T_{k+1}-1]$, we have $x_{t+1}=x_{t}^+(k)$.
So from (\ref{eq:one-step:app}), we get
\begin{align}\label{leq:phasek}
\sum_{t=T_{k}}^{T_{k+1}-1} (f(x_{t})-f(x_{*})) & \leq \sum_{t=T_{k}}^{T_{k+1}-1}  \pa{{\gamma_{k}L\over 2\sqrt{T}} +{L\sqrt{T}\over 2 \gamma_{k}} \pa{\|x_{*}-x_{t}\|^2-\|x_{t+1}-x_{*}\|^2}}\nonumber\\
 &=  {\gamma_{k}L\over 2 \sqrt{T}}  (T_{k+1}-T_{k}) + {L\sqrt{T}\over 2 \gamma_{k}} \pa{\|x_{T_{k}}-x_{*}\|^2-\|x_{T_{k+1}}-x_{*}\|^2}\,,
\end{align}
with the convention that $\sum_{t=T_{k}}^{T_{k}-1}=0$.
We bound the second term in the right-hand side of (\ref{leq:phasek}), with (\ref{eq:dist-opt:app})
\[\|x_{T_{k}}-x_{*}\|^2 \leq \|x_{1}-x_{*}\|^2 + \Gamma^2_{T_{k}-1}+{\gamma_{k}\over T}= \|x_{1}-x_{*}\|^2 + \Gamma^2_{T_{k}}\,,\]
and for the third term, we combine (\ref{eq:inflate}) with a triangular inequality to get
\[\|x_{T_{k+1}}-x_{*}\|^2 \geq \cro{\|x_{1}-x_{*}\|- \|x_{T_{k+1}}-x_{1}\|}^2_{+} \geq  \cro{\|x_{1}-x_{*}\|- {B_{T_{k+1}}}}^2_{+}\,,\]
where we used the condensed notation $B_{T_{k+1}}:=B_{T_{k+1}}(k_{{T_{k+1}}-1})=B_{T_{k+1}}(k)$.
Plugging these two upper and lower bounds in (\ref{leq:phasek}),
 and applying the simple inequality
\begin{equation}\label{eq:phi}
\Delta^2-[\Delta-B]_{+}^2\leq 2\Delta B, \quad \text{for all}\ \Delta,B\geq 0,
\end{equation}
we get
\begin{align*}
\sum_{t=T_{k}}^{T_{k+1}-1} (f(x_{t})-f(x_{*}))
& \leq  {\gamma_{k}L\over 2 \sqrt{T}}  (T_{k+1}-T_{k}) + {L\sqrt{T}\over 2 \gamma_{k}} \pa{\|x_{1}-x_{*}\|^2+\Gamma_{T_{k}}^2-\cro{\|x_{1}-x_{*}\|- B_{T_{k+1}}}^2_{+}}\\
&\leq {L\sqrt{T}\over 2} \pa{\gamma_{k} {T_{k+1}-T_{k}\over T}+{\Gamma_{T_{k}}^2\over \gamma_{k}}+{2B_{T_{k+1}}\over \gamma_{k}}\|x_{1}-x_{*}\| }\,,
\end{align*}
using $\Gamma_{T_{k}}^2\leq \gamma_{k-1}^2 T_{k}/T/\leq \gamma_{k-1}\gamma_{k}/2$ for $k\geq 1$, 
we get
\begin{align*}
\sum_{t=T_{k}}^{T_{k+1}-1} (f(x_{t})-f(x_{*}))&\leq {L\sqrt{T}\over 2} { \pa{\gamma_{k} {T_{k+1}-T_{k}\over T}+{\gamma_{k-1}\over 2}+
{2B_{T_{k+1}}\over \gamma_{k}}{\|x_{1}-x_{*}\| }}
}\,.\nonumber\\
\end{align*}

{\bf Bounding the total regret.}
Summing the above inequality over $k$, we get the upper-bound on the total regret
\begin{align}
\sum_{t=1}^{T} (f(x_{t})-f(x_{*}))
&\leq {L\sqrt{T}\over 2} {\sum_{1\leq k\leq k_{T}} \pa{\gamma_{k_{T}} {T_{k+1}-T_{k}\over T}+{\gamma_{k-1}\over 2} +
{2B_{T_{k+1}}\over \gamma_{k}}\|x_{1}-x_{*}\| }}\nonumber\\
&\leq {L\sqrt{T}\over 2} { \pa{{3\gamma_{k_{T}}\over 2} +\sum_{1\leq k\leq k_{T}}
{2B_{T_{k+1}}\over \gamma_{k}}\|x_{1}-x_{*}\| }}\,. \label{leq:regret}
\end{align}
We point out that the bound (\ref{leq:regret}) is valid for any choice of $B_{t+1}(k)$.
Let us treat apart the two cases.

\begin{itemize}
\item {\bf Simple case: $B^{\text{simple}}_{t+1}(k)= 3\gamma_{k}$ .} Using that $k_{T}\leq k^*$ in this case,
\[B_{T_{k+1}}=B_{T_{k+1}}(k)=3\gamma_{k}\,\,,\]
and recalling the upper bound (\ref{leq:app:kstar}) on $k^*$ and $\gamma_{k^*}$,
we get from  (\ref{leq:regret})
\begin{align*}
\sum_{t=1}^{T} (f(x_{t})-f(x_{*})) &\leq {L\sqrt{T}\over 2} \pa{{3\over 2}\gamma_{k^*} +
  \sum_{1\leq k\leq k^*} 6\|x_{1}-x_{*}\| }\\
&= {L\sqrt{T}\over 2} \pa{{{3\over 2}\gamma_{k^*} + 6k^*\|x_{1}-x_{*}\| }}\\
&\leq L\sqrt{T} \pa{3\|x_{1}-x_{*}\| \log_{2}\pa{ {2D_{\gamma_{0}}\over \gamma_{0}}} + 2 D_{\gamma_{0}}}\,.
\end{align*}

\item {\bf \FreeAdaGrad case: $B_{t+1}(k)= {2 \gamma_{k} \over  k^{1/2}} + \sqrt{ \Gamma^2_{t}+ {\gamma^2_{ k}\over T}}$ .}
We have proved in this case that $k_{T}\leq \bar k$ with $\bar k$ upper bounded by (\ref{borne:bark}).
 Combining (\ref{borne:bark}) and (\ref{leq:app:kstar}),
the first term in the right-hand side of (\ref{leq:regret}) can be readily bounded by
\[{3\over 4} \gamma_{\bar k} \leq {3\over 4} \gamma_{k^*+0.5 \log_{2}(k^*)+1.25} \leq 4
 D_{\gamma_{0}}\sqrt{\log_{2} \pa{ {2 D_{\gamma_{0}}\over \gamma_{0}}}}\,.\]
The last term in the right-hand side of (\ref{leq:regret}), can be bounded as follows.
We notice that
\[\Gamma_{T_{k+1}-1}^2+{\gamma^2_{k_{T_{k+1}-1}}\over T}=\Gamma_{T_{k+1}}^2\,,\]
so we have
\[\sum_{1\leq k\leq \bar k} {B_{T_{k+1}}\over \gamma_{k}} = \sum_{1\leq k\leq \bar k} \pa{{2\over \sqrt{k}}+{\Gamma_{T_{k+1}}\over \gamma_{k}}}\leq 4 \sqrt{\bar k}+ \sum_{1\leq k\leq \bar k}{\Gamma_{T_{k+1}}\over \gamma_{k}}\,.\]
For the last term, we observe that
\begin{align*}
\sum_{1\leq k\leq \bar k} {\Gamma_{T_{k+1}}\over \gamma_{k}}&=\sum_{1\leq k\leq \bar k}\gamma_{k}^{-1}\sqrt{\sum_{j\leq k} \gamma_{j}^2\Delta T_{j}/T}
\ \leq \sum_{1\leq k\leq \bar k}\sum_{j\leq k}\gamma_{k}^{-1}\gamma_{j}\sqrt{\Delta T_{j}/T}\\
&\leq \sum_{1\leq j\leq \bar k} \gamma_{j}\sqrt{\Delta T_{j}/T} \sum_{k:k\geq j} \gamma_{k}^{-1}= 2\sum_{1\leq j\leq \bar k} \sqrt{\Delta T_{j}/T}\ \leq 2\sqrt{\bar k}\,,
\end{align*}
where the last inequality follows from Cauchy Schwarz.
Then, plugging these bounds in (\ref{leq:regret}), and using that $\bar k=1$ when $k^*=1$, and
\[\bar k\leq k^*+0.5 \log_{2}(k^*)+1.25\leq 2 k^*,\quad \text{for}\ k^*\geq 2\,,\]
we get from (\ref{leq:app:kstar})
\begin{align*}
\sum_{t=1}^T (f(x_{t})-f(x_{*}))&\leq {L\sqrt{T}} \cro{{3\over 4}\gamma_{\bar k}+6\|x_{1}-x_{*}\| \sqrt{\bar k}}\\
&\leq {10 D_{\gamma_{0}}L\sqrt{T}}  \sqrt{2\log_{2} \pa{ {2D_{\gamma_{0}}\over \gamma_{0}}}}\,.
\end{align*}
which concludes the proof of Theorem \ref{thm:warm-up}.
\end{itemize}

\newpage
\section{Proofs for Section~\ref{sec:meta}}\label{app:meta}

\begin{proof}[Proof of Theorem~\ref{thm:any_h}]
First of all, observe that the algorithm in question can be written as
\begin{align*}
    x_{t+1} = \Proj_{\Theta}\pa{x_t - \frac{\gamma_{k_t}}{h_t}g_t}\enspace,
\end{align*}
where we recall that $(h_t)_{t \geq 1}$ is assumed to be non-decreasing and positive.
As before, we denote by $[T_k, T_{k+1} - 1]$ the interval where $k_t = k$. In particular, $T_{k_{T}+1} - 1=T$.
On the interval $[T_k, T_{k+1} - 1]$, the algorithm is simply AdaGrad (slightly modified) started from the point $x_{T_{k}}$ and with the final point at $x_{T_{k+1}}$. Thus, within each phase, we can apply the analysis of the AdaGrad that we recall and slightly adapt in Appendix \ref{app:adagrad}, page \pageref{app:adagrad}.
The proof closely follows that of the warm-up setup: observing that
\begin{align*}
   \sum_{t = 1}^T(f(x_t) - f(x_*)) = \sum_{k = 1}^{k_T} \sum_{t = T_k}^{T_{k+1}-1}(f(x_t) - f(x_*))\enspace,
\end{align*}
\begin{enumerate}
    \item We start with one phase analysis, using the results of Appendix \ref{app:adagrad}, page \pageref{app:adagrad}, which contains Lemma~\ref{lem:ada_main} displayed in the main body;
    \item Then, we sum-up the total regret over $k_T$ phases, using the previous analysis, and bound the key quantities;
\end{enumerate}

\paragraph{One phase analysis.}
Fix some $k \leq k_T$ and assume that the the $k$th phase is non-empty, that is, $T_{k+1} > T_k$. Thus, in view of the above discussion, Lemma~\ref{lem:regret}, page \pageref{lem:regret}, yields
\begin{equation}
\label{eq:start_adagrad}
\begin{aligned}
        \sum_{t = T_k}^{T_{k+1}-1}(f(x_t) - f(x_*))
        &\leq
        {h_{T_{k+1}}}\Bigg({\frac{\norm{x_{T_k} - x_*}^2 {-} \norm{x_{T_{k+1}} - x_*}^2}{2\gamma_k}
        + \frac{\gamma_k}{2}{\underbrace{\sum_{t = T_k}^{T_{k+1} - 1}\frac{\|g_t\|^2}{h^2_t}}_{ = \frac{\Gamma^2_{T_{k+1}} - \Gamma^2_{T_k}}{\gamma_k^2}}}}\Bigg)\\
        &=
        {h_{T_{k+1}}}\parent{\frac{\norm{x_{T_k} - x_*}^2 - \norm{x_{T_{k+1}} - x_*}^2}{2\gamma_k}
        +
        \frac{\Gamma^2_{T_{k+1}} {-} \Gamma^2_{T_k}}{2\gamma_k}}\enspace.
\end{aligned}
\end{equation}
Note that by design $\norm{x_{T_{k+1}} - x_*} \geq \big[\norm{x_1 - x_*} - B_{T_{k+1}}(k)\big]_+$. Furthermore, iteratively applying Lemma~\ref{lem:deviation} by phases, we deduce that
\begin{align*}
    \norm{x_{T_k} - x_*}^2 \leq \|x_1 - x^*\|^2 + \Gamma_{T_k}^2\enspace.
\end{align*}
That is, we have
\begin{align*}
    \frac{\norm{x_{T_k} - x_*}^2 - \norm{x_{T_{k+1}} - x_*}^2}{2\gamma_k}
    &\leq
    \frac{\norm{x_{1} - x_*}^2 - \big[\norm{x_{1} - x_*} - B_{T_{k+1}}(k)\big]^2_+}{2\gamma_k} + \frac{\Gamma_{T_k}^2}{2\gamma_k}\enspace.
\end{align*}
Furthermore, recalling that $\Delta^2 - [\Delta - B]^2_+ \leq 2 \Delta B$, the above can be further bounded as
\begin{align}
    \label{eq:meta00}
    \frac{\norm{x_{T_k} - x_*}^2 - \norm{x_{T_{k+1}} - x_*}^2}{2\gamma_k}
    &\leq
    \norm{x_1 - x_*}\frac{B_{T_{k+1}}(k)}{\gamma_k} + \frac{\Gamma_{T_k}^2}{2\gamma_k}\enspace.
\end{align}
Substitution of~\eqref{eq:meta00} into~\eqref{eq:start_adagrad}, yields
\begin{align}
    \label{eq:start_adagrad2}
        \sum_{t = T_k}^{T_{k+1}-1}(f(x_t) - f(x_*))
        &\leq
        {h_{T_{k+1}}}\norm{x_1 - x_*}\frac{B_{T_{k+1}}(k)}{\gamma_k} + h_{T_{k+1}}\frac{\Gamma^2_{T_{k+1}}}{2\gamma_k}\enspace.
\end{align}

\paragraph{Summing up over phases.} Summing up all the inequalities~\eqref{eq:start_adagrad2} for $k_T$ phases, we obtain
\begin{align}
    \label{eq:meta0}
    \sum_{t = 1}^{T}(f(x_t) - f(x_*)) \leq \norm{x_1 - x_*}\underbrace{\sum_{k \leq k_T}\parent{{h_{T_{k+1}}}\frac{B_{T_{k+1}}(k)}{\gamma_k}}}_{=:\mathsf{I}} + \underbrace{\sum_{k \leq k_T}{ h_{T_{k+1}}\frac{\Gamma^2_{T_{k+1}}}{2\gamma_k}}}_{=:\mathsf{II}}\enspace.
\end{align}

\paragraph{Bounding the sum of $h\tfrac{\Gamma}{2\gamma}$ terms ($\mathsf{I}$).}
Observe that, by definition of thereof,
\begin{align}
    \Gamma^2_{T_{k+1}} = \sum_{j \leq k}\gamma_j^2 \pa{\sum_{t = T_j}^{T_{j+1}-1} \frac{\|g_t\|^2}{h^2_t}}\enspace.
\end{align}
Hence, using trivial bound $h_{T_{k+1}} \leq h_{T+1}$, we deduce
\begin{equation}
\label{eq:meta1}
\begin{aligned}
    \mathsf{I} = \sum_{k \leq k_T}h_{T_{k+1}}\frac{\Gamma^2_{T_{k+1}}}{2\gamma_k}
    &\leq
    \frac{h_{T+1}}{2}\sum_{k \leq k_T}\frac{\Gamma^2_{T_{k+1}}}{\gamma_k}\\
    &=
    \frac{h_{T+1}}{2}\sum_{k \leq k_T}\sum_{j \leq k}\gamma_k^{-1}\gamma_j^2 \pa{\sum_{t = T_j}^{T_{j+1}-1} \frac{\|g_t\|^2}{h^2_t}}\\
    &=
    \frac{h_{T+1}}{2}\sum_{j \leq k_T}\sum_{k \geq j}\gamma_k^{-1}\gamma_j^2 \pa{\sum_{t = T_j}^{T_{j+1}-1} \frac{\|g_t\|^2}{h^2_t}}\\
    &\leq
    {h_{T+1}}\sum_{j \leq k_T}\gamma_j \pa{\sum_{t = T_j}^{T_{j+1}-1} \frac{\|g_t\|^2}{h^2_t}}\\
    &\leq
    h_{T+1}\gamma_{k_T}\sum_{t = 1}^{T}\frac{\|g_t\|^2}{h^2_t}\enspace,
\end{aligned}
\end{equation}
where the penultimate inequality is due to the fact that $\sum_{k \geq j}2^{-k} \leq 2^{-j+1}$ and the last one holds since $\gamma_k \leq \gamma_{k_T}$.

\paragraph{Bounding the sum of $B$ terms ($\mathsf{II}$).}
We observe that by definition of $B_t(k)$, we have
\[B_{T_{k+1}}(k)=B_{T_{k+1}-1+1}(k)= {2\gamma_{k}\over \sqrt{k}}+\Gamma^2_{T_{k+1}}\enspace.\]
Hence, the term of interest is bounded as
\begin{align*}
    \mathsf{II} = \sum_{k \leq k_T}h_{T_{k+1}}\frac{B_{T_{k+1}}(k)}{\gamma_k}
   & =
    {2\sum_{k \leq k_T}\frac{h_{T_{k+1}}}{\sqrt{k}}}+  {\sum_{k \leq k_T}h_{T_{k+1}}\gamma_k^{-1}\Gamma_{T_{k+1}}}\\
    &\leq 4h_{T+1}\sqrt{{k}_T}+{h_{T+1}\sum_{k \leq k_T}\gamma_k^{-1}\Gamma_{T_{k+1}}}    \enspace.
\end{align*}
For the third term, similarly to the previous paragraph, but additionally invoking Jensen's inequality, we can write
\begin{align*}
    \sum_{k \leq k_T}\gamma_k^{-1}\Gamma_{T_{k+1}}
    &=
    k_T\sum_{k \leq k_T}\frac{1}{k_T}\sqrt{\sum_{j \leq k}\gamma^2_{j}\gamma_k^{-2}\parent{\sum_{t = T_j}^{T_{j+1}-1}\frac{\|g_t\|^2}{h^2_t}}}\\
    &\leq
    \sqrt{k_T}\sqrt{\sum_{j \leq k_T}\sum_{k \geq j}\gamma^2_{j}\gamma_k^{-2}\parent{\sum_{t = T_j}^{T_{j+1}-1}\frac{\|g_t\|^2}{h^2_t}}}\\
    &\leq
    {2 \sqrt{k_T}\over \sqrt{3}}\sqrt{\sum_{t = 1}^{T} \frac{\|g_t\|^2}{h^2_t}}\enspace,
\end{align*}
where in the last inequality we used the fact that  $\sum_{k = a}^b 2^{-2k} \leq \tfrac{4}{3}2^{-2a}$. 
Thus, overall, we have
\begin{align}
\label{eq:meta2}
    \mathsf{II} = \sum_{k \leq k_T}h_{T_{k+1}}\frac{B_{T_{k+1}}(k)}{\gamma_k} \leq 2h_{T+1}\sqrt{k_T}\parent{2 + \sqrt{\frac{1}{3}\sum_{t = 1}^{T}\frac{\|g_t\|^2}{h^2_t}}}
\end{align}

\paragraph{The end (combining bounds for $\mathsf{I}$ and $\mathsf{II}$).}
Substituting~\eqref{eq:meta1} and~\eqref{eq:meta2} into~\eqref{eq:meta0}, we deduce that for any non-decreasing $(h_t)_{t\geq 1}$
\begin{equation*}
    \sum_{t = 1}^{T}(f(x_t) - f(x_*))
    \leq
    h_{T+1}\pa{{2\norm{x_1 - x_*}\sqrt{k_T}\pa{2 + \sqrt{\frac{1}{3}\sum_{t = 1}^{T}\frac{\|g_t\|^2}{h^2_t}}}}
    + \gamma_{k_T}\sum_{t = 1}^{T}\frac{\|g_t\|^2}{h^2_t}}\enspace.\qedhere
\end{equation*}
\end{proof}

\subsection{Basic analysis for AdaGrad and Proof of Lemma~\ref{lem:ada_main}}\label{app:adagrad}
In this section we extend the standard analysis of $\AdaGrad$ for our purposes and prove Lemma~\ref{lem:ada_main} restated below.
Throughout, we consider the following algorithm for $t \geq 1$
\begin{align}
    \label{eq:ada_grad_def}
    x_{t + 1} = \Proj_{\Theta}\parent{x_t - \frac{\gamma}{h_t} g_t}\enspace,
\end{align}
where $g_t \in \partial f(x_t)$, $S_t = \sum_{s = 1}^t\norm{g_s}^2$ and $(h_t)_{t \geq 1}$ is non-decreasing and positive, and $x_{1} \in \bbR^d$.

We start with some elementary results.
\begin{lemma}
    \label{lem:basic}
     For all $t \geq 1$ and all $x_1 \in \bbR$
     \begin{align*}
         0 \leq \frac{2\gamma}{h_t}(f(x_t) - f(x_*)) \leq \norm{x_t - x_*}^2 - \norm{x_{t+1} - x_*}^2 + \frac{\gamma^2}{h^2_t}\norm{g_t}^2\enspace.
     \end{align*}
\end{lemma}
\begin{proof}
By the property of projection
\begin{align}
    \norm{x_{t+1} - x_*}^2
    &\leq
    \norm{x_t - x_*}^2 - \frac{2\gamma}{h_t}\scalar{x_t - x_*}{g_t} + \frac{\gamma^2}{h^2_t}\norm{g_t}^2\nonumber\\
    \label{eq:known_L1}
    &\leq
    \norm{x_t - x_*}^2 - \frac{2\gamma}{h_t}(f(x_t) - f(x_*)) + \frac{\gamma^2}{h^2_t}\norm{g_t}^2\enspace,
\end{align}
where we used the fact that $f$ is convex.
The result follows after re-arranging.
\end{proof}
Lemma~\ref{lem:basic} applied iteratively yields the following result.
\begin{lemma}
    \label{lem:deviation}
    For all $T \geq 1, \bar\gamma > 0$ and all $x_1 \in \bbR$, Algorithm~\eqref{eq:ada_grad_def} satisfies
    \begin{align*}
    &\norm{\Proj_{\Theta}\parent{x_{T} - \frac{\bar\gamma}{h_T}g_T} - x_*}^2 \leq \norm{x_1 - x_*}^2 + \gamma^2\sum_{t = 1}^{T-1}\frac{\norm{g_t}^2}{h^2_t} + \bar{\gamma}^2\frac{\|g_T\|^2}{h^2_T}\enspace.
\end{align*}
\end{lemma}

Finally, we are in position to prove Lemma~\ref{lem:ada_main} brought up in the main body of the paper.
\begin{lemma}[Restated Lemma~\ref{lem:ada_main} from Section~\ref{sec:meta}]
\label{lem:regret}
For all $T > 1$ and all $x_1 \in \bbR^d$ we have
    \begin{align*}
        \sum_{t = 1}^{T-1}(f(x_t) - f(x_*)) \leq
        {h_T}\pa{{\frac{\norm{x_{1} - x_*}^2 - \norm{x_{T} - x_*}^2}{2\gamma}}
        + \frac{\gamma}{2}\sum_{t = 1}^{T-1}\frac{\norm{g_t}^2}{h^2_t}}\enspace.
    \end{align*}
\end{lemma}
\begin{proof}
    Using Lemma~\ref{lem:basic}, we deduce that
\begin{align}
    \label{eq:adagrad_start}
    \sum_{t = 1}^{T-1}(f(x_t) - f(x_*))
    &\leq
    \frac{1}{2\gamma}\sum_{t = 1}^{T-1}h_t\parent{{\norm{x_{t} - x_*}^2 - \norm{x_{t+1} - x_*}^2}} +  \frac{\gamma}{2}\sum_{t = 1}^{T-1}\frac{\norm{g_t}^2}{h_t}\enspace.
\end{align}
Let us bound the first sum on the right hand side, adding and subtracting $h_{t+1}\norm{x_{t+1} - x_*}^2$ and using telescoping summation, we obtain
\begin{equation}
    \label{eq:adagrad1}
\begin{aligned}
    \sum_{t = 1}^{T-1}h_t\parent{{\norm{x_{t} - x_*}^2 - \norm{x_{t+1} - x_*}^2}}
    &=
    h_1{\norm{x_{1} - x_*}^2 - h_T\norm{x_{T} - x_*}^2}\\
    &\phantom{=}+ \sum_{t = 1}^{T-1}\parent{h_{t+1} - h_t}\norm{x_{t+1} - x_*}^2\enspace.
\end{aligned}
\end{equation}
Furthermore, by Lemma~\ref{lem:deviation} with $\bar{\gamma} = \gamma$ and the fact that $(h_t)_{t \geq 1}$ is non-decreasing, we get
\begin{align*}
     \sum_{t = 1}^{T-1}\parent{h_{t+1} - h_t}\norm{x_{t+1} - x_*}^2
     &\leq
     \sum_{t = 1}^{T-1}\parent{h_{t+1} - h_t}\parent{\norm{x_1 - x_*}^2 + \gamma^2 \sum_{s = 1}^{t}\frac{\norm{g_s}^2}{h^2_s}}\\
     &\leq
     \parent{h_{T} - h_1}\norm{x_1 - x_*}^2
     + \gamma^2 \sum_{t = 1}^{T-1}\parent{h_{t+1} - h_t}\sum_{s = 1}^{t}\frac{\norm{g_s}^2}{h^2_s}\enspace.
\end{align*}
For the second term in the above bound, we can write
    \begin{align*}
        \sum_{t = 1}^{T-1}\parent{h_{t+1} - h_t}\sum_{s = 1}^{t}\frac{\norm{g_s}^2}{h^2_s}
        &=
        \sum_{s = 1}^{T - 1}\frac{\norm{g_s}^2}{h^2_s}\sum_{t = s}^{T-1}\parent{h_{t+1} - h_t}\\
        &=
        \sum_{s = 1}^{T - 1}\frac{\norm{g_s}^2}{h^2_s}\pa{h_T - h_s}\\
        &=
        h_T\sum_{t = 1}^{T - 1}\frac{\norm{g_t}^2}{h^2_t} - \sum_{t = 1}^{T - 1}\frac{\norm{g_t}^2}{h_t}\enspace.
    \end{align*}

Substitution of the above into the penultimate inequality yields
\begin{equation}
    \label{eq:adagrad_improved}
\begin{aligned}
    \sum_{t = 1}^{T-1}\parent{h_{t+1} {-} h_t}\norm{x_{t+1} - x_*}^2 \leq (h_T {-} h_1)\|x_1 - x_*\|^2
    {+} \gamma^2\pa{h_T\sum_{t = 1}^{T - 1}\frac{\norm{g_t}^2}{h^2_t}
    {-} \sum_{t = 1}^{T - 1}\frac{\norm{g_t}^2}{h_t}}\,,
\end{aligned}
\end{equation}
Substituting~\eqref{eq:adagrad_improved} into~\eqref{eq:adagrad1}, we deduce that

\begin{align*}
    \sum_{t = 1}^{T-1}h_t\parent{{\norm{x_{t} - x_*}^2 - \norm{x_{t+1} - x_*}^2}} &\leq
     h_T\parent{\norm{x_1 - x_*}^2 - \norm{x_{T} - x_*}^2}
    \\ &\phantom{\leq}
    + \gamma^2 h_T\sum_{t = 1}^{T-1}\frac{\norm{g_t}^2}{h^2_t} - \gamma^2\sum_{t = 1}^{T - 1}\frac{\norm{g_t}^2}{h_t}\enspace.
\end{align*}
Combination of the above with~\eqref{eq:adagrad_start} concludes the proof.
\end{proof}

\subsection{Proofs of corollaries in Section~\ref{sec:meta}}\label{app:cor:proof}
In this section we provide proofs of four corollaries presented in Section~\ref{sec:meta}.

\begin{proof}[Proof of Corollary~\ref{cor:main}]
    Substituting our choice of $h_t$ into Theorem~\ref{thm:any_h}, we prove in Lemma~\ref{lem:sums2} in Appendix~\ref{app:aux}, page \pageref{lem:sums2}, that
    \begin{align*}
        \sum_{t = 1}^T \frac{\|g_t\|^2}{h^2_t} \leq \log(\log(\e(1 + S_T)))\enspace.
    \end{align*}
    Substituting the above into Theorem~\ref{thm:any_h} and using~\eqref{eq:bound_on_k_T}, we deduce that
    \begin{align*}
        R_T
        \leq
        \sqrt{(S_{T+1} + 1)\log(\e(S_{T+1}+1))}
        \bigg[
        &\sqrt{8}\norm{x_1 - x_*}\sqrt{\log_2\pa{\frac{D_{\gamma_0}}{\gamma_0}} + 1}
        \pa{2 {+} \sqrt{\frac{1}{3}\log(\log(\e(1 + S_T)))}}\\
        &+ 5D_{\gamma_0}\log(\log(\e(1 + S_T)))\sqrt{\log_2\pa{\frac{D_{\gamma_0}}{\gamma_0}} + 1}
        \bigg]\enspace.
  \end{align*}
  The proof is concluded after re-arranging and using $2ab \leq a^2 + b^2$.
\end{proof}

\begin{proof}[Proof of Corollary~\ref{cor:adagrad_real}]
Theorem~\ref{thm:any_h} and Lemma~\ref{lem:number_of_phases} (rather Eq.~\eqref{eq:bound_on_k_T}) and Lemma~\ref{lem:sums3} give
\begin{align*}
    \sum_{t = 1}^T(f(x_t) - f(x_*))
    &\leq
    \sqrt{S_{T+1}\log_2\pa{\frac{2D_{\gamma_0}}{\gamma_0}}}\Bigg[2\norm{x_1 - x_*}\sqrt{2}\pa{2 {+} \sqrt{\frac{1}{3}\log\pa{\e\pa{\frac{S_T}{\|g_1\|^2}}}}}\\
    &\phantom{\leq\sqrt{S_{T+1}\log_2\pa{\frac{2D_{\gamma_0}}{\gamma_0}}}}
    + 5D_{\gamma_0}\log\pa{\e\pa{\frac{S_T}{\|g_1\|^2}}} \Bigg]\\
    &\leq
    D_{\gamma_0}\sqrt{S_{T+1}\log_2\pa{\frac{2D_{\gamma_0}}{\gamma_0}}}\Bigg[2\sqrt{2}\pa{2 {+} \sqrt{\frac{1}{3}\log\pa{\e\pa{\frac{S_T}{\|g_1\|^2}}}}}\\
    &\phantom{\leq
    \sqrt{S_{T+1}\log_2\pa{\frac{2D_{\gamma_0}}{\gamma_0}}}}
    + 5\log\pa{\e\pa{\frac{S_T}{\|g_1\|^2}}} \Bigg]\enspace.
  \end{align*}
  The proof is concluded after re-arranging and using $2ab \leq a^2 + b^2$.
\end{proof}

\begin{proof}[Proof of Corollary~\ref{cor:adagrad_like}]
    From Lemma~\ref{lem:sums4} we have
    \begin{align*}
        \sum_{t = 1}^T \frac{\|g_t\|^2}{\varepsilon + S_t} \leq \log\pa{1 + \frac{S_T}{\varepsilon}}\enspace.
    \end{align*}
    Hence, substituting the above into Theorem~\ref{thm:any_h} and using~\eqref{eq:bound_on_k_T}, we obtain
    \begin{align*}
    R_T
    \leq
    \sqrt{\varepsilon +  S_{T+1}}\sqrt{\log_2\pa{\frac{D_{\gamma_0}}{\gamma_0}} + 1}\Bigg[&\sqrt{8}\norm{x_1 - x_*}\pa{2 {+} \sqrt{\frac{1}{3}\log\pa{1 + \frac{S_T}{\varepsilon}}}} + 5D_{\gamma_0}\log\pa{1 + \frac{S_T}{\varepsilon}} \Bigg]\enspace.
  \end{align*}
  The proof is concluded after re-arranging and using $2ab \leq a^2 + b^2$.
\end{proof}

As promised in Section~\ref{sec:meta}, Theorem~\ref{thm:warm-up} of Appendix~\ref{app:warm-up} can be obtained as a corollary of Theorem~\ref{thm:any_h}.
\begin{corollary}
    \label{cor:simple}
    Under assumptions of Theorem~\ref{thm:any_h}, with $f$ an $L$-Lipschitz function.
    Setting $h_t = L\sqrt{T}$ and $D_{\gamma_0} = \|x_1 - x_*\|\vee\gamma_0$, Algorithm~\ref{algo:PGDloglogMath} satisfies
    \begin{align*}
        \sum_{t = 1}^T(f(x_t) - f(x_*))
        \leq
        12.3D_{\gamma_0}L\sqrt{T\log_2\pa{\frac{2D_{\gamma_0}}{\gamma_0}}}\enspace.
  \end{align*}
\end{corollary}
\begin{proof}[Proof of Corollary~\ref{cor:simple}]
    Substituting $h \equiv L\sqrt{T}$ into Theorem~\ref{thm:any_h} gives
    \begin{align*}
    R_T
    &\leq
    L\sqrt{T}\pa{2\norm{x_1 - x_*}\sqrt{k_T}\pa{2 {+} \sqrt{\frac{1}{3}}} + \gamma_{k_T}}
    +L\sqrt{T}\frac{\gamma_{k_T}}{2}\,,
    \end{align*}
    Eq.~\eqref{eq:bound_on_k_T} applied to the above, yields
    \begin{align*}
    R_T
    \leq
    L\sqrt{T}\Bigg[\sqrt{8}&\norm{x_1 - x_*}\sqrt{\log_2\pa{\frac{D_{\gamma_0}}{\gamma_0}} + 1}\pa{2 {+} \sqrt{\frac{1}{3}}} + 5D_{\gamma_0}\sqrt{\log_2\pa{\frac{D_{\gamma_0}}{\gamma_0}} + 1}\Bigg]\\
    &\leq
    L\sqrt{T}\Bigg[\underbrace{\sqrt{8}\pa{2 {+} \sqrt{\frac{1}{3}}} + 5
    }_{\leq 12.3} \Bigg]D_{\gamma_0}\sqrt{\log_2\pa{\frac{D_{\gamma_0}}{\gamma_0}} + 1}\enspace.
    \end{align*}
    The proof is concluded.
\end{proof}

\newpage

\section{Analysis for stochastic PGD in Section~\ref{sec:SGD}: proof of Theorem~\ref{thm:SGD}}
\label{app:SGD}

\begin{proof}[Proof of Theorem~\ref{thm:SGD}]
We recall that $\ell(\delta) = 1 \vee \log(\log_2(T) / \delta)$.
As in all the previous sections, we denote by $[T_{k}, T_{k+1}-1]$, the interval where $k_t = k$. The regret of the algorithm can be expressed as
\begin{align*}
    \sum_{t = 1}^T (f(x_t) - f(x_*)) = \sum_{k = 1}^{k_T}\underbrace{\vphantom{\sum_{T_k}}(f(x_{T_{k}}) - f(x_{*}))}_{=:\texttt{T}_1(k)}+ \sum_{k = 1}^{k_T}\underbrace{\sum_{t = T_k+1}^{T_{k+1} - 1}(f(x_t) - f(x_{*}))}_{=:\texttt{T}_2(k)}\enspace.
\end{align*}
We are going to apply Lemma~\ref{lem:SGD_high_proba} of Appendix~\ref{app:high_proba} on page~\pageref{app:high_proba}, to the second term and provide deterministic bound on the first one. First let us explain the reason of such separation of terms. Observe that for $t = T_{k}$
\begin{align*}
    x_{t + 1} = \Proj_{\Theta}\pa{x_t - \frac{\gamma_{k_t}}{L\sqrt{T\ell(\delta / (1 + k_{t-1})^2)}}g_t}\enspace.
\end{align*}
Meanwhile, for $t \in  [T_{k}+1,T_{k+1}-1]$ we have
\begin{align*}
     x_{t + 1} = \Proj_{\Theta}\pa{x_t - \frac{\gamma_{k_t}}{L\sqrt{T\ell(\delta / (1 + k_{t})^2)}}g_t}\enspace.
\end{align*}
That is, the $x_{T_{k} + 1}$ step is outside the pattern and requires additional splitting.
The proof proceeds as follows
\begin{enumerate}
    \item First we give deterministic bound on $\texttt{T}_1(k)$ terms using Lipschitzness of the objective function $f$;
    \item Then we use Lemma~\ref{lem:SGD_high_proba} to bound each $\texttt{T}_2(k)$;
    \item Finally, we show that $k_T$ is still bounded by $k^*$ as in the warmup analysis.
\end{enumerate}

\paragraph{Analysis for $\texttt{T}_1(k)$.} Since $f$ is assumed to be Lipschitz, we can write
\begin{align*}
    \texttt{T}_1(k) \leq L \norm{x_{T_k} - x_*}\enspace.
\end{align*}
Then, simply by the triangle inequality and property of the Euclidean projection, we deduce that for all $k \geq 2$
\begin{align*}
    \norm{x_{T_k} - x_*}
    &\leq
    \norm{x_{T_{k-1}+1} - x_{T_{k-1}}} + \norm{x_{T_{k-1}} - x_*} + (T_k - T_{k-1} - 1)\frac{\gamma_{k - 1}}{\sqrt{T\ell(\delta / k^2)}}\\
    &\leq
    \norm{x_{T_{k-1}+1} - x_{T_{k-1}}} + \norm{x_{T_{k-1}} - x_*} + (T_k - T_{k-1} - 1)\frac{\gamma_{k_T - 1}}{\sqrt{T\ell(\delta )}}\enspace.
\end{align*}
Furthermore, we have
\begin{align*}
    \norm{x_{T_{k-1}+1} - x_{T_{k-1}}} \leq \frac{\gamma_{k - 1}}{\sqrt{T\ell(\delta / (k - 1)^2)}} \leq \frac{\gamma_{k_T - 1}}{\sqrt{T\ell(\delta)}}\enspace.
\end{align*}
Hence, we have
\begin{align*}
    \norm{x_{T_k} - x_*}
    &\leq
    \norm{x_{T_{k-1}} - x_*} + (T_k - T_{k-1})\frac{\gamma_{k_T - 1}}{\sqrt{T\ell(\delta )}}\enspace.
\end{align*}
Unfolding the above recursion, we deduce that
\begin{align*}
    \norm{x_{T_k} - x_*}
    &\leq
    \norm{x_{T_{1}} - x_*} + T_k\frac{\gamma_{k_T - 1}}{\sqrt{T\ell(\delta )}} \leq \norm{x_{1} - x_*} + 0.5\gamma_{k_T}\sqrt{T}\enspace.
\end{align*}
We conclude that
\begin{equation}
\label{eq:SGD_first_term}
\begin{aligned}
\sum_{k = 1}^{k_T}(f(x_{T_{k}}) - f(x_{*}))
&\leq Lk_T\pa{\norm{x_{1} - x_*} + 0.5\gamma_{k_T}\sqrt{T}}\\
&\leq L\sqrt{T\ell_{T}(\delta/(1 + k_T)^2)}k_T\pa{\norm{x_{1} - x_*} + 0.5\gamma_{k_T}}\enspace.
\end{aligned}
\end{equation}

\paragraph{Analysis for $\texttt{T}_2(k)$.} Let us first fix the high-probability event on which we are going to work.
Note that $\rho=T_{k}+1$ is a stopping time and $T_{k+1} - 1-\rho \leq T$. Thus, by Lemma~\ref{lem:SGD_high_proba} with probability at least $1 - \delta / (1 + k)^2$, it holds that
\begin{align*}
    \texttt{T}_2(k) \leq L\sqrt{T\ell_{T}(\delta/(1+k)^2)}\pa{\frac{\gamma_k}{2} + \frac{\|x_{T_k+1} - x_*\|^2 - \|x_{T_{k+1}} - x_*\|^2}{2\gamma_k} + 10\|x_{T_k+1} - x_*\| + 68\gamma_k}\enspace.
\end{align*}
We observe that
\begin{align*}
    \|x_{T_{k}+1} - x_*\|^2 - \|x_{T_{k+1}} - x_*\|^2
    &\leq
    \|x_{T_k + 1} - x_*\|^2 - \left[\|x_{T_k + 1} - x_*\| - \|x_{T_k + 1} - x_{T_{k+1}}\|\right]_+^2\\
    &\leq
    2\|x_{T_k + 1} - x_*\|\|x_{T_k + 1} - x_{T_{k+1}}\|\enspace.
\end{align*}
Let us bound each term of the product. By the design of the rule,
\begin{align*}
    \|x_{T_k + 1} - x_{T_{k+1}}\| \leq \|x_{T_k + 1} - x_1\| + \|x_1 - x_{T_{k+1}}\| \leq 38\gamma_{k} + 38\gamma_{k} \leq 76\gamma_k\enspace,
\end{align*}
where $B = 28$.
Furthermore, by Lemma~\ref{lem:SGD_high_proba} and the fact that the $\|x_{T_k + 1} - x_{T_k}\| \leq \gamma_{k}$ for all $k \geq 1$
\begin{equation}
\label{eq:SGD_triangle}
\begin{aligned}
    \|x_{T_k+1} - x_*\|
   & \leq
    \|x_{T_{k}} - x_*\| + 2\gamma_{k-1}\\
    &\leq
    \|x_{T_{k-1}+1} - x_*\| + 18\gamma_{k-1}\\
  &  \leq
    \|x_{2} - x_*\| + 18\sum_{j = 1}^{k-1}\gamma_{j}\\
    &\leq
     \|x_{2} - x_*\| + 18\gamma_k\\
    &\leq \|x_{1} - x_*\| + 18\gamma_k + \gamma_1\\
    &\leq \|x_{1} - x_*\| + 19\gamma_k
    \enspace.
\end{aligned}
\end{equation}
Thus, we have shown that with probability at least $1 - \delta / (1 + k)^2$
\begin{align*}
    \texttt{T}_2(k)
    &\leq
    L\sqrt{T\ell_{T}(\delta/(1+k)^2)}\pa{\frac{\gamma_k}{2} + 76\pa{\|x_{1} - x_*\| + 19\gamma_k} + 10(\|x_{1} - x_*\| + 19\gamma_k) + 68\gamma_k}\\
    &\leq
    L\sqrt{T\ell_{T}(\delta/(1 + k_T)^2)}\pa{{86}\|x_{1} - x_*\| + (259 + 38^2)\gamma_k}\\
    &\leq
    L\sqrt{T\ell_{T}(\delta/(1 + k_T)^2)}\pa{86\|x_{1} - x_*\| + 1703\gamma_{k_T}}\enspace.
\end{align*}

Overall, by the union bound, we have with probability at least $1 - \sum_{k = 1}^{\infty} \delta/(1 + k)^2 \geq 1 - \delta$
\begin{align}
    \label{eq:SGD_second_term}
    \sum_{k = 1}^{k_T}\texttt{T}_2(k) \leq L\sqrt{T\ell_{T}(\delta/(1 + k_T)^2)}k_T\pa{86\|x_{1} - x_*\| + 1703\gamma_{k_T}}
\end{align}

\paragraph{Regret bound}
Putting together~\eqref{eq:SGD_first_term} and~\eqref{eq:SGD_second_term}, we obtain with probability $1-\delta$
\begin{align}
    \label{eq:SGD_almost_final}
     \sum_{t = 1}^T(f(x_t) - f(x_*))
     &\leq
     L\sqrt{T\ell_{T}(\delta/(1 + k_T)^2)}k_T\pa{87\|x_{1} - x_*\| + 1704\gamma_{k_T}}.
\end{align}

\paragraph{Bounding the number of phases}
It remains to bound the number of phases $k_T$.
Fix some $t \geq 1$ and $k \geq k_{t-1}$.
Observe that for any $k \geq 1$, $\|x^+_t(k) - x_t\| \leq \gamma_k$. Then thanks to Lemma~\ref{lem:SGD_high_proba} and Eq.~\eqref{eq:SGD_triangle} (which hold on exact same event that we consider in~\eqref{eq:SGD_almost_final}), we have
\begin{align*}
    \|x^+_t(k) - x_*\|
    &\leq
    \|x_t - x_*\|
    + \gamma_{k}
    \leq
    \|x_{T_{k_t} + 1} - x_*\|
    + 16\gamma_{k_t} + \gamma_k\\
    &\leq
    \|x_{1} - x_*\| + 35\gamma_{k_t} + \gamma_k
    \leq
    \|x_{1} - x_*\| + 36\gamma_{k}
    \enspace.
\end{align*}
Hence, for all $k \geq k_{t-1}$
\begin{align*}
    \|x^+_t(k) - x_1\| \leq 2\|x_{1} - x_*\| + 36\gamma_{k} \leq 2\gamma_{k^*} + 36\gamma_{k}\enspace.
\end{align*}
Implying that $k_T \leq k^*$.

\paragraph{Concluding}
In view of the bound $k_T \leq k^* \leq \log_2(2D_{\gamma_0} / \gamma_0)$ and $\gamma_{k^*} \leq 2D_{\gamma_0}$, we conclude that with probability at least $1 - \delta$
\begin{align*}
     \sum_{t = 1}^T(f(x_t) - f(x_*))
     &\leq
     3500LD_{\gamma_0}\sqrt{T\ell_{T}(\delta/(1 + k^*)^2)}\log_2(2D_{\gamma_0} / \gamma_0)
     \enspace.
\end{align*}
Note that the constant $3500$ is certainly extremely pessimistic as we did not attempt to optimize it.

\end{proof}
\subsection{High probability bound}
\label{app:high_proba}

We slightly adapt a version of the Bernstein-Freedman inequality, derived in~\cite[Corolary 16]{Bern}, improving $\log(T)$ dependency to $\log\log(T)$.
\begin{lemma}[a version of the Bernstein-Freedman inequality by~\cite{Bern}]
\label{lm:Ber}
Let $X_1, \, X_2, \, \ldots$ be a martingale difference with respect
to the filtration $\mathcal{F} = (\mathcal{F}_s)_{s \geq 0}$ and with increments bounded in absolute values
by $K$. For all $t \geq 1$, let
\[
\mathfrak{S}_t^2 = \sum_{\tau = 1}^t \E\bigl[ X_\tau^2 \,\big|\, \mathcal{F}_{\tau-1} \bigr]
\]
denote the sum of the conditional variances of the first $t$ increments. Then,
for all $\delta \in (0,1)$ and $T\geq 1$, with probability at least $1-\delta$,
\[
\max_{t\leq T}\sum_{\tau = 1}^t X_\tau \leq 2 \mathfrak{S}_T\sqrt{ \ln\pa{ \frac{\log_{2}(2T)}{\delta}}} + 3 K \ln \pa{\frac{\log_{2}(2T)}{\delta}}\,.
\]
\end{lemma}

\begin{proof}
Let $X^*_{T}=\max_{t \leq T} \sum_{s=1}^t X_{s}$. For $k\geq 1$ we have
\begin{align*}
\lefteqn{\mathbb{P}\cro{X^*_{T} > \sqrt{4(\mathfrak{S}_T^2+K^2)\ell}+\sqrt{2}K\ell/3;\ K^{-2}\mathfrak{S}_T^2\in [2^{k-1}-1,2^{k}]}} \\
&\leq \mathbb{P}\cro{X^*_{T} > \sqrt{2^{k+1}K^2\ell}+\sqrt{2}K\ell/3;\ K^{-2}\mathfrak{S}_T^2\in [2^{k-1}-1,2^{k}]}\\
&\leq  \mathbb{P}\cro{X^*_{T} > \sqrt{2^{k+1}K^2\ell}+\sqrt{2}K\ell/3;\ K^{-2}\mathfrak{S}_T^2\leq 2^{k}} \ \leq \ e^{-\ell},
\end{align*}
where the last inequality follows from Lemma 15 in~\cite{Bern}.
Since $0\leq K^{-2}\mathfrak{S}_T^2 \leq T$, we take a union bound over $k=1,\ldots,\lceil{\log_{2}(T)}\rceil$ and notice that
\[\sqrt{4(\mathfrak{S}_T^2+K^2)\ell}+\sqrt{2}K\ell/3 \leq 2\sqrt{\mathfrak{S}_T^2\ell}+3K\ell\,\,.\qedhere\]
\end{proof}

The next result is a version of Lemma~\ref{lem:ada_main}, that was used to analyze deterministic setup, which accounts for the stochasticity.
\begin{lemma}
    \label{lem:SGD_high_proba}
    Let $\rho$ be a bounded stopping time with respect to the filtration $\mathcal{F}$ of the stochastic gradients.
    Let $\delta \in (0, 1)$, $T \geq 1$, $x'_1 \in \mathcal{F}_{\rho}$ and consider $\ell_{T}(\delta):= 1\vee \log(\log_{2}(2T)/\delta)$,
    \begin{align*}
        x'_{t+1} = x'_t - \frac{\gamma}{L\sqrt{T\ell_{T}(\delta)}}g_{\rho+t}\enspace.
    \end{align*}
    Assume that $T, \delta$ are such that $T \geq 1$, then with probability at least $1 - \delta$ we have for all $\tau \leq T$
    \begin{align*}
        &
        \|x'_\tau - x_*\| \leq \|x'_1 - x_*\| + 16\gamma \enspace,\\
        &\sum_{t = 1}^{\tau}(f(x'_t) - f_*) \leq L\sqrt{T\ell_{T}(\delta)}\pa{\frac{\gamma}{2} + \frac{\|x'_1 - x_*\|^2 - \|x'_{\tau+1} - x_*\|^2}{2\gamma} + 10\|x'_1 - x_*\| + 68\gamma}
    \end{align*}
    simultaneously.
\end{lemma}

\begin{proof}
To simplify the expressions, we drop the primes, writing $x_{t}$ for $x'_{t}$, and we set
 $\eta = \tfrac{\gamma}{L\sqrt{T\ell_{T}(\delta)}}$. Using classical analysis of projected gradient descent, we obtain
\begin{align}
    \sum_{t \leq \tau} \scalar{g_{\rho+t}}{x_t - x_*} \leq \frac{\eta L^2}{2} + \frac{\|x_1 - x_*\|^2 - \|x_{\tau+1} - x_*\|^2}{2\eta}\enspace.
\end{align}
Introducing the following martingale difference
\begin{align*}
    X_t = {\scalar{\nabla f(x_t) - g_{t+\rho}}{{x_t - x_*}}}\enspace,
\end{align*}
we deduce from the above, and the fact that $\tau \leq T$
\begin{align}
    \label{eq:SGD0}
    \sum_{t = 1}^\tau(f(x_t) - f(x_*)) \leq \frac{\eta L^2 T}{2} + \frac{1}{2\eta} \left(\|x_1 - x_*\|^2 - \|x_{\tau+1} - x_*\|^2\right) + \sum_{t = 1}^\tau X_t\enspace.
\end{align}

\paragraph{Dealing with randomness.}
Now we are in position to apply Freedman-Bernstein inequality recalled in Lemma~\ref{lm:Ber}.

To this end, we need to bound $X_t$ and get an appropriate expression for $\mathfrak{S}_t$.
First observe that each martingale difference satisfies
\begin{align}
    \label{eq:bound_mart_diff}
    |X_t| \leq 2L\|x_t - x_*\|\qquad\text{almost surely}\enspace.
\end{align}
Furthermore, by the triangle inequality, property of projection and the fact that $\|g_t\| \leq L$ almost surely, we obtain
\begin{align*}
    \norm{x_{t} - x_*} \leq \norm{x_{t - 1} - \eta g_{\rho+t-1} - x_*} \leq \norm{x_{t - 1} - x_*} + \eta L \leq \norm{x_1 - x_*} + \eta L T, \quad \forall t \leq T+1\enspace.
\end{align*}
Hence,
\begin{align*}
    |X_t| \leq K := 2L\norm{x_1 - x_*} + 2L^2 T \eta, \qquad \forall t \leq T+1.
\end{align*}

The conditional variance $\mathfrak{S}^2_T$ can be bounded using~\eqref{eq:bound_mart_diff} as
\begin{align*}
    \mathfrak{S}_T \leq 2L \sqrt{\sum_{t = 1}^T \|x_t - x_*\|^2} \leq 2L\sqrt{T} \max_{t \leq T} \|x_t - x_*\| \qquad\text{almost surely}\,.
\end{align*}

Thus, invoking Lemma~\ref{lm:Ber}, for any $\delta \in (0, 1)$ with probability at least $1 - \delta$,
\begin{align}
    \label{eq:martingale_bound}
 \max_{\tau \leq T}   \sum_{t = 1}^\tau X_t \leq &4L\max_{t \leq T}\|x_t - x_*\|\sqrt{T\ell_{T}(\delta)} + 6L(\|x_1 - x_*\| + L\eta T)\ell_{T}(\delta)\enspace.
\end{align}
From now on, we work on this event which holds with probability at least $1-\delta$.

Substituting~\eqref{eq:martingale_bound} into~\eqref{eq:SGD0}, we get for all $\tau \leq T$
\begin{equation}
\label{eq:key_sgd}
\begin{aligned}
    \sum_{t = 1}^\tau(f(x_t) - f(x_*))
    \leq
    &\frac{\eta L^2 T}{2} + \frac{\|x_1 - x_*\|^2 - \|x_{\tau+1} - x_*\|^2}{2\eta}
    +
    4L\Phi_T\sqrt{T\ell_{T}(\delta)}
    \\
    &+ 6L(\|x_1 - x_*\| + L\eta T)\ell_{T}(\delta)\enspace,
\end{aligned}
\end{equation}
where $\Phi_T = \max_{t \leq T}\|x_t - x_*\|$.

\paragraph{Bounding the trajectory.}
Observing that the left hand side of~\eqref{eq:key_sgd} is non-negative and that it holds for all $\tau \leq T$, we deduce
\begin{align*}
    \max_{0\leq \tau\leq T}\|x_{\tau+1} - x_*\|^2
    &\leq
    \parent{\|x_1 - x_*\|^2 + 12L\eta\ell_{T}(\delta)\|x_1 - x_*\|} + \parent{1 + 12\ell_{T}(\delta)}{\eta^2 L^2 T} \\
    &+ 8L\eta\Phi_T\sqrt{T\ell_{T}(\delta)}\enspace.
\end{align*}
Solving the above inequality, we deduce that
\begin{align*}
 \Phi_T &\leq
    \sqrt{{(\|x_1 - x_*\| + 6L\eta\ell_{T}(\delta))^2} + \parent{1 + 28\ell_{T}(\delta)}{\eta^2 L^2 T}}
    + 4L\eta\sqrt{T\ell_{T}(\delta)}\\
    &\leq
    \|x_1 - x_*\| + 6L\eta\ell_{T}(\delta) + \eta L\sqrt{\parent{1 + 28\ell_{T}(\delta)}{T}}
    +
   4 L\eta\sqrt{T\ell_{T}(\delta)}\enspace.
\end{align*}


Substituting the value of $\eta$, we further deduce that
\begin{equation}
\label{eq:sgd_k1}
\begin{aligned}
    \max_{t \leq T} \|x_t - x_*\|
    &\leq
    \|x_1 - x_*\| + \gamma\pa{\sqrt{\frac{1 + 28\ell_{T}(\delta)}{\ell_{T}(\delta)}} + 4+6\sqrt{\ell_{T}(\delta)\over T}}\\
    &\leq
    \|x_1 - x_*\|  + \gamma\underbrace{\pa{\sqrt{29} + 4+ 6}}_{\leq 15.5}\enspace.
\end{aligned}
\end{equation}

\paragraph{Bounding the regret}
On the other hand, substituting \eqref{eq:martingale_bound} into~\eqref{eq:SGD0}, we obtain
\begin{align*}
    \sum_{t = 1}^\tau(f(x_t) - f_*) \leq &\frac{\eta L^2 T}{2} + \frac{\|x_1 - x_*\|^2 - \|x_{\tau+1} - x_*\|^2}{2\eta} \\
    &+4L\max_{t \leq T}\|x_t - x_*\|\sqrt{T\ell_{T}(\delta)} + 6(L\|x_1 - x_*\| + L^2\eta T)\ell_{T}(\delta)\enspace.
\end{align*}
Substitution of~\eqref{eq:sgd_k1} into the above inequality, yields
\begin{align*}
    \sum_{t = 1}^\tau(f(x_t) - f_*) \leq &\frac{\eta L^2 T}{2} + \frac{\|x_1 - x_*\|^2 - \|x_{\tau+1} - x_*\|^2}{2\eta} \\
    &+4L\pa{\|x_1 - x_*\| + 15.5\gamma}\sqrt{T\ell_{T}(\delta)}
    + 6(L\|x_1 - x_*\| + L^2\eta T)\ell_{T}(\delta)\enspace.
\end{align*}
Recalling that $\eta = \gamma / (L\sqrt{T\ell_{T}(\delta)})$ and using some rough bounds, we deduce that
\begin{align*}
    \sum_{t = 1}^\tau(f(x_t) - f_*)
    \leq
    &L\sqrt{T\ell_{T}(\delta)}\pa{\frac{\gamma}{2} + \frac{\|x_1 - x_*\|^2 - \|x_{\tau+1} - x_*\|^2}{2\gamma} + \underbrace{(4+6)}_{=10}\|x_1 - x_*\| + \underbrace{(4\times 15.5+6)}_{=68}\gamma}\enspace.
\end{align*}
The proof is complete.
\end{proof}

\newpage
\section{On a relation with~\cite{mcmahan2012no}}\label{app:MS12}
In case where there exists a known bound $\|g_{t}\|\leq L$ on the norms of the subgradients,  MacMahan and Streeter~\cite{mcmahan2012no} propose to tune the step size of gradient descent with a scheme based on a reward doubling argument and cold-restart.
Their theory works in a setup of unconstrained online convex optimization with $L$-bounded subgradients.
Since we do not require Lipschitz functions, and we additionally handle the projection step, the two results cannot be directly compared.
Nevertheless, there are some similarities and, in a specific instantiation of our Algorithm~\ref{algo:PGDloglogMath}, we recover that of~\cite{mcmahan2012no}.

Below, we sketch the relation between the two, considering the setting of  MacMahan and Streeter~\cite{mcmahan2012no},
where the norms of the subgradients are bounded by some known $L$ and the optimization is unconstrained, i.e. $\Theta = \bbR^d$.
We also assume that the time horizon $T$ is known, since unknown $T$ is handled in MacMahan and Streeter~\cite{mcmahan2012no} by a time-doubling trick.
Using our notation, their analysis starts with a simple bound
\begin{align*}
\sum_{t=1}^T (f_{t}(x_{t})-f_{t}(x_{*}))&\leq \sum_{t=1}^T \langle g_{t}, x_{t}-x_{*}\rangle =\sum_{t=1}^T \langle g_{t}, x_{1}-x_{*}\rangle+\sum_{t=1}^T \langle g_{t}, x_{t}-x_{1}\rangle\\
&\leq \|\underbrace{\sum_{t=1}^T g_{t}}_{=: G_{T}}\| \|x_{1}-x_{*}\| - \underbrace{\sum_{t=1}^T \langle g_{t}, x_{1}-x_{t}\rangle}_{=:Q_{T}} = \|G_{T}\|  \|x_{1}-x_{*}\| - Q_{T}\enspace.
\end{align*}
They observe, using a duality argument, that it is sufficient to show that
\begin{equation}\label{eq:Q-condition}
Q_{T} \geq a^{-1} \exp\pa{\|G_{T}\|/(bL\sqrt{T})}-c,
\end{equation}
in order to derive
\[\sum_{t=1}^T (f_{t}(x_{t})-f_{t}(x_{*})) \leq b \|x_{1}-x_{*}\|L\sqrt{T} \log\pa{ab \|x_{1}-x_{*}\|L\sqrt{T}}+c\enspace.\]
The principle of their algorithm is to perform gradient descent by phases and, during a phase, to track the reward $Q_{t}$ relative to this phase, and restart with a doubled step-size when the condition $Q_{t}> \eta L^2 t$ is met. This step-size doubling ensures that the Condition~(\ref{eq:Q-condition}) is met at the time horizon $T$.

Let us relate this algorithm with a specific instantiation of our Algorithm~\ref{algo:PGDloglogMath}.
When the algorithm is the simple Gradient Descent (GD), that does not involve the projection step, we have $x_{T+1}-x_{1}=-\eta G_{T}$ for the GD with a fixed step size $\eta$.
Hence, it holds that
\begin{align*}
Q_{T}&= \sum_{t=1}^T \langle g_{t}, x_{1}-x_{t}\rangle = \eta \sum_{t=1}^T \langle g_{t}, G_{t-1}\rangle
 = {\eta\over 2} \sum_{t=1}^T \pa{\|G_{t}\|^2-\|G_{t-1}\|^2-\|g_{t}\|^2}\\
&=  {\eta\over 2}\pa{\|G_{T}\|^2-\sum_{1}^T \|g_{t}\|^2}={1\over 2\eta} \|x_{T+1}-x_{1}\|^2- {1\over 2\eta} \tilde \Gamma^2_{T+1}
\end{align*}
where $ \tilde \Gamma^2_{T+1} = \sum_{1}^T \eta^2\|g_{t}\|^2$.
Thus, if in Algorithm~\ref{algo:PGDloglogMath} we allow cold restarts (the exact thing that we want to avoid), then the condition
\[ \|x^+_{T}(\eta)-x_{1}\|^2 \leq 2 \eta^2 L^2 T + \tilde \Gamma^2_{T}+\eta^2 \|g_{T}\|^2\]
is equivalent to their doubling condition
\[ Q_{T} \leq \eta L^2 T\enspace.\]
In our notation, the algorithm of MacMahan and Streeter~\cite{mcmahan2012no} corresponds to a variant of Algorithm~\ref{algo:PGDloglogMath} with
\[\tilde B_{t+1}(k)= \sqrt{\pa{2+{\|g_{t}\|^2\over L^2 T}}\gamma^2_{k}+\Gamma^2_{t}-\Gamma^2_{T_{k}-1}}\enspace,\]
with the major difference that a cold-restart is performed when $k_t$ is increased, and the minor difference that step-size doubling happens after (and not before) the condition $\norm{x_{t}^+(k)-x_{1}}\leq \tilde B_{t+1}(k)$ is broken.

\newpage
\section{Auxiliary results}
\label{app:aux}
Let $(a_t)_{t \geq 1}$ be a non-negative sequence, and $S_t = \sum_{\tau = 1}^t a_\tau$.
For any concave function $F$ on $[0,+\infty)$, we have
\begin{equation}\label{leq:sum:adagrad}
\sum_{t=1}^T a_{t} F'(S_{t}) \leq \sum_{t=1}^T (F(S_{t})-F(S_{t-1}))=F(S_{T})-F(0).
\end{equation}
Applying (\ref{leq:sum:adagrad}) with $F(x)=2\sqrt{x}$, $F'(x)=1/\sqrt{x}$, we get the following bound.

\begin{lemma}
    \label{lem:sums}
    Let $(a_t)_{t \geq 1}$ be a non-negative sequence and $S_t = \sum_{\tau = 1}^t a_\tau$, then for all $\eps > 0$
    \begin{align*}
        \sum_{t = 1}^T \frac{a_t}{\sqrt{S_t}} \leq 2\sqrt{S_T}\enspace.
    \end{align*}
\end{lemma}

The inequality (\ref{leq:sum:adagrad}) with $F(x)=\log(\log(e(1+x)))$, $F'(x)=((1+x)\log(e(1+x)))^{-1}$ gives the next lemma (the first inequality follows directly from Lemma~\ref{lem:sums}).

\begin{lemma}
    \label{lem:sums2}
    Let $(a_t)_{t \geq 1}$ be a non-negative sequence and $S_t = \sum_{\tau = 1}^t a_\tau$, then for all $\eps > 0$
    \begin{align*}
        &\sum_{t = 1}^T \frac{a_t}{\sqrt{(S_t+1)\log(\e(1+S_t))}} \leq 2\sqrt{S_T}\enspace,\\
        &\sum_{t = 1}^T \frac{a_t}{(S_t + 1)\log(\e(1+S_t))} \leq \log(\log(\e(1 + S_T)))\enspace.
    \end{align*}
\end{lemma}

Using~\eqref{leq:sum:adagrad}, with $F(x) = \log(x)$, we deduce that
\begin{lemma}
    \label{lem:sums3}
    Let $(a_t)_{t \geq 1}$ be a non-negative sequence and $S_t = \sum_{\tau = 1}^t a_\tau$, then
    \begin{align}
        \sum_{t = 1}^T \frac{a_t}{S_t} = 1 + \sum_{t = 2}^T\frac{a_t}{S_t} \leq 1 + \log(S_T/S_1)\enspace.
    \end{align}
\end{lemma}

Using~\eqref{leq:sum:adagrad}, with $F(x) = \log(\varepsilon+x)$, we deduce that
\begin{lemma}
    \label{lem:sums4}
    Let $(a_t)_{t \geq 1}$ be a non-negative sequence and $S_t = \sum_{\tau = 1}^t a_\tau$, then
    \begin{align}
        \sum_{t = 1}^T \frac{a_t}{\varepsilon+S_t}  \leq  \log(1+S_T/\varepsilon)\enspace.
    \end{align}
\end{lemma}


\begin{lemma}\label{lem:bark}
 Let us define $\bar k$ as the smallest integer fulfilling $\bar k^{-1/2}2^{\bar k}\geq 2^{k^*}$. Then \[\bar k\leq k^*+0.5 \log_{2}(k^*)+1.25\,\,.\]
\end{lemma}
\begin{proof}
We observe that $\bar k=1$ for $k^*=1$.  We will prove that
\begin{equation*}
    \bar k \leq \lceil k^*+0.5 \log_{2}(k^*)+ 0.25 \rceil \leq k^*+0.5 \log_{2}(k^*)+1.25\enspace.
\end{equation*}
For proving the first inequality, we only need to prove that $2^y/\sqrt{y} \geq 2^{k^*}$ for $y=k^*+0.5 \log_{2}(k^*)+ 0.25$. Plugging the value of $y$ and taking the square, we get
\[{2^y\over \sqrt{y}} \geq 2^{k^*} \Leftrightarrow {2^{1/4}2^{k^*}\sqrt{k^*}\over \sqrt{k^*+0.5 \log_{2}(k^*)+ 0.25}} \geq 2^{k^*} \Leftrightarrow \sqrt{2}\, k^* \geq k^*+0.5 \log_{2}(k^*)+ 0.25\,\,.\]
%
So, all we need is to prove by induction that
\begin{align*}
    \sqrt{2}\,k^* \geq k^*+0.5 \log_{2}(k^*)+ 0.25\enspace.
\end{align*}
For $k^* = 2$, the inequality holds.
By induction hypothesis, for $k^* \geq 2$
\begin{align*}
    \sqrt{2}(k^*+1) \geq (k^*+1) + \frac{1}{2} \log_{2}(k^*)  + \frac{1}{4}+ \sqrt{2}-1\enspace.
\end{align*}
Thus, it suffices to show that \begin{align*}
    \frac{1}{2} \log_{2}(k^*) + \sqrt{2} - 1 \geq \frac{1}{2} \log_{2}(k^* + 1), \quad \text{for all $k^*\geq 2$}\enspace,
\end{align*}
or equivalently
\begin{align*}
    \frac{1}{2}\log_2\pa{1 + \frac{1}{k^*}} + 1 \leq \sqrt{2},\quad \text{for all $k^*\geq 2$}\enspace,
\end{align*}
to prove the induction.
The concavity of $\log_{2}$ ensures that $\log_2(1 + x) \leq x / \ln(2)$ for all $x > -1$. Thus, for all $k^* \geq 2$
\begin{align*}
    \frac{1}{2}\log_2\pa{1 + \frac{1}{k^*}} + 1 \leq \frac{1}{2 \ln(2) k^*} + 1 \leq \frac{1}{4 \ln(2)} + 1 \leq \sqrt{2}\enspace,
\end{align*}
which concludes the proof of Lemma~\ref{lem:bark}.
\qedhere

\end{proof}






















\end{document}